\documentclass{article}

\usepackage{amsmath}
\usepackage{graphics} 			
\usepackage{epsfig} 			
\usepackage{mathptmx} 			
\usepackage{times} 				
\usepackage{amsthm} 			
\usepackage{amsfonts}
\usepackage{amstext}
\usepackage{amssymb}
\usepackage{subfigure}
\usepackage[english]{babel}

\newtheorem{theorem}{Theorem}
\newtheorem*{theorem*}{Theorem}

\newtheorem{problem}{Problem}
\newtheorem{proposition}{Proposition}

\DeclareMathOperator{\tr}{tr}

\def\real{\mbox{\rm I\kern-0.18em R}}
\def\Tiny{\fontsize{6.8pt}{6.8pt}\selectfont}

\title{\LARGE \bf Synergy--Based Hand Pose Sensing:\\ Optimal Glove Design\thanks{This work is supported by the European Commission under CP
grant no. 248587, THE Hand Embodied, within the FP7-ICT-2009-4-2-1
program Cognitive Systems and Robotics.}}

\author{Matteo Bianchi\thanks{The Interdept. Research Center ``En\-ri\-co Pi\-ag\-gio'', U\-ni\-ver\-si\-ty of Pisa, via Diotisalvi 2, 56100 Pisa, Italy. {\tt\footnotesize	 m.bianchi,p.salaris,bicchi@centropiaggio.unipi.it}}, \and Paolo Salaris\footnotemark[2], \and Antonio Bicchi\footnotemark[2] \thanks{Department of Advanced Robotics, Istituto Italiano di Tecnologia, via Morego, 30, 16163 Genova, Italy} }
\date{ }

\begin{document}

\maketitle
\thispagestyle{empty}
\pagestyle{empty}

\begin{abstract}
In this paper we study the problem of improving human hand pose sensing device performance by exploiting the knowledge on how humans most frequently use their hands in grasping tasks. In a companion paper we studied the problem of maximizing the reconstruction accuracy of the hand pose from partial and noisy data provided by any given pose sensing device (a sensorized ``glove'') taking into account statistical {\em a priori} information. In this paper we consider the dual problem of how to design pose sensing devices, i.e.~how and where to place sensors on a glove, to get maximum information about the actual hand posture. We study the continuous case, whereas individual sensing elements in the glove measure a linear combination of joint angles, the discrete case, whereas each measure corresponds to a single joint angle, and the most general hybrid case, whereas both continuous and discrete sensing elements are available. The objective is to provide, for given {\em a priori} information and fixed number of measurements, the optimal design minimizing in average the reconstruction error. Solutions relying on the geometrical synergy definition as well as gradient flow-based techniques are provided. Simulations of reconstruction performance show the effectiveness of the proposed optimal design.
\end{abstract}

\section{Introduction}
\label{Int}

This paper investigates the problem of estimating the
posture of human hands using sensing devices, and how to improve their
performance based on the knowledge on how humans most frequently use
their hands. Similarly to the companion paper \cite{Bianchi_etalI},
this work is motivated by studies on the human hand in grasping tasks
\cite{Santelloart} suggesting hand posture representations of
increasing complexity (``synergies''), which allow to reduce the
number of Degrees of Freedom (DoFs) to be used according to the
desired level of approximation. In \cite{Bianchi_etalI}, we analyzed
the role of the {\em a priori} information for pose hand
reconstructions by using given sensing devices, and showed that
acceptable reconstruction results can be obtained even in presence of
insufficient and inaccurate sensing data.

In this work, we extend the analysis to consider the optimal design of
sensing ``gloves'', i.e.~devices for hand pose reconstruction based on
measurements of few geometric features of the hand. The problem we
consider is to find the distribution of a number of sensing elements
of limited accuracy so as to provide, together with the {\em a priori}
information, the optimal design which minimizes in probability the
reconstruction error. The problem becomes particularly relevant when
limits on the production costs of sensing gloves introduce constraints
limiting both the number and the quality of sensors. In these cases, a
careful design of sensor distribution is instrumental to obtain good
performance.
\begin{figure}[!t]
\centering \subfigure{\label{fig:ContinuousSensing}\includegraphics[width=0.405\columnwidth]{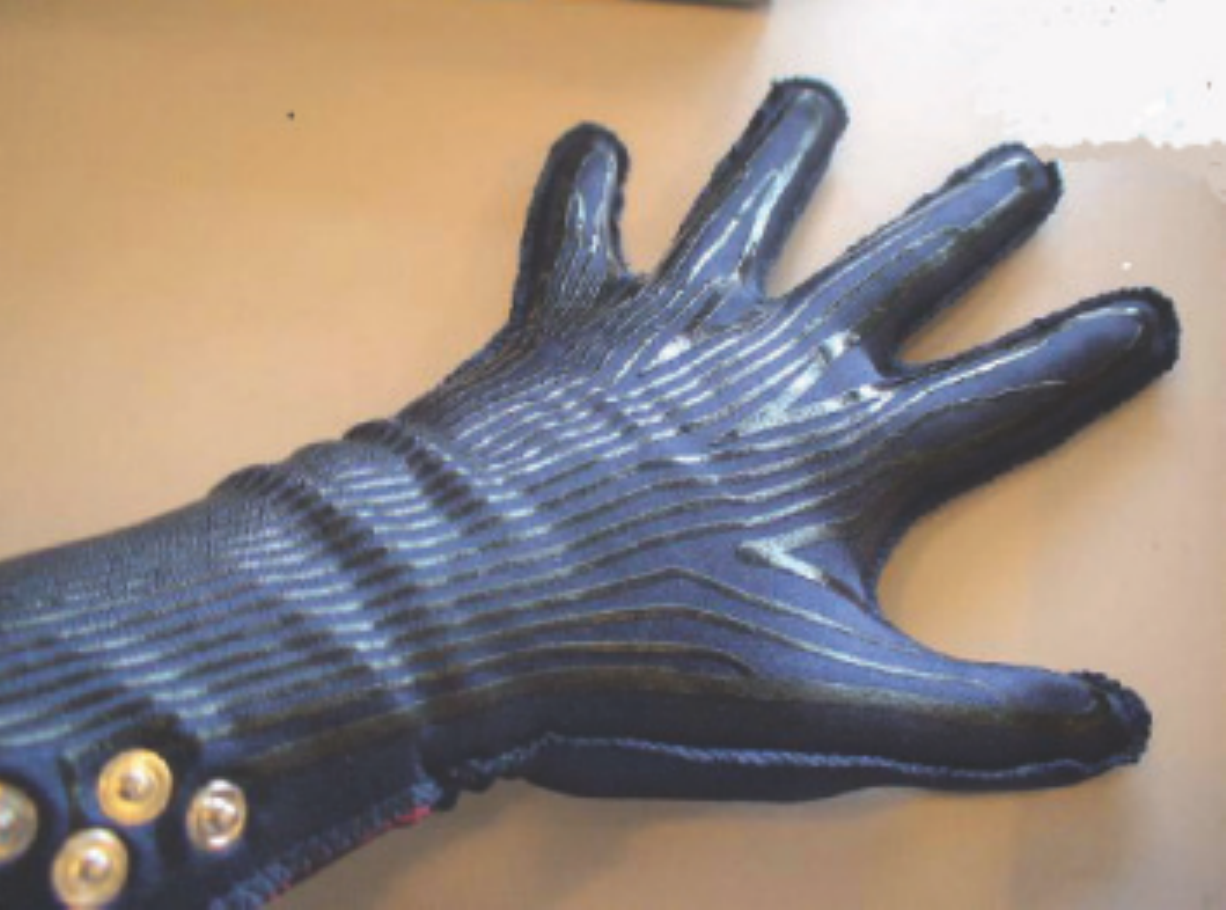}}
  \qquad \subfigure{\label{fig:DiscreteSensing}\includegraphics[width=0.4\columnwidth]{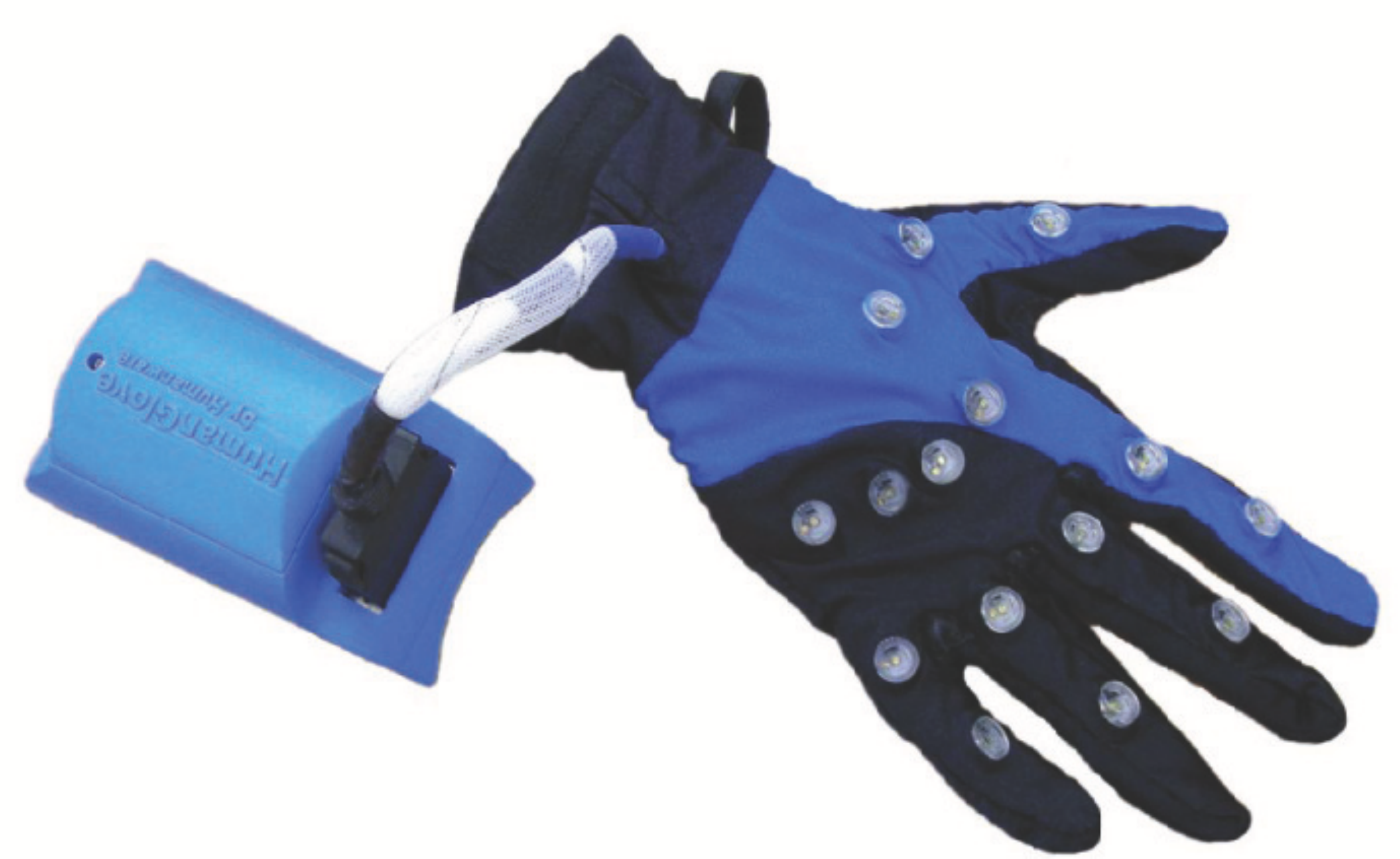}}
  \caption{Examples of continuous and discrete sensorized gloves.  On
      the left, a sensing glove based on conductive elastomer sensor
      strips printed on fabric, each measuring a linear combination of
      joint angles \cite{Tognetti}. On the right, the Humanglove (image
      courtesy by Humanware s.r.l.~(www.hmw.it/)), using individual joint
      angle sensors.}
\label{fig:guanti}
\end{figure}

Optimal experimental design represents a challenging, widely discussed
topic in literature \cite{BookOptimalDesign}. Among all optimal design
criteria, Bayesian methods are ideally suited to contribute to
experimental design and error statistics minimization, when some
information is available prior to experimentation (see
e.g.~\cite{Bayesian1,Bayesian2,Bicchican} for a review).  On the
contrary, non Bayesian criteria are adopted when a linear Gaussian
hypothesis is not fulfilled and/or when the designer's primary concern
is to minimize worst-case sensing errors rather than error
statistics. Criteria on explicit worst-case/deterministic bounds on
the errors and tools from the theory of optimal
worst-case/deterministic estimation and/or identification are
discussed e.g.~in~\cite{Helmicki,Tempo,Bicchican,BicchiOpt}.

However, most of these approaches refer to cases with a number of basic sensors which is redundant or at least equal to the number of variables to be estimated. Moreover, no previous example of application to the peculiar problem of exploiting the a priori psychophysical information on the structure of human hand embodiment
for under--sensorized gloves has ever been reported. In~\cite{Sturman93} an investigation of ``whole-hand'' interfaces for the control of complex tasks is presented, along with the description, design, and evaluation of whole-hand inputs, based on empirical data from users. In~\cite{Edmin} authors discussed the properties, advantages, and design aspects associated with piezoelectric materials
for sensing glove design, in an application where the device is used
as a keyboard. Finally,~\cite{Chang} authors explored how to methodically select a minimal set of hand pose features from optical marker data for grasp recognition. The objective is to determine marker locations on the hand surface that is appropriate for grasp classification of hand poses. All the aforementioned approaches rely on experimental or qualitative observations: from actual sensor data, locations that provide the largest and most useful information on the system are chosen.

In this paper, we investigate in depth the problem of obtaining the
optimal distribution of sensors minimizing in probability the
reconstruction error of hand poses. We adopt a classical Bayesian
approach to minimize the {\em a posteriori} covariance matrix norm and
hence, to maximize the information on the real hand posture available by
the glove measurement.

The {\em a posteriori} covariance matrix,
$P_p=P_o-P_oH^T(HP_oH^T$$+R)^{-1}HP_o$, which directly depends on the
sensor design through the measurement matrix $H$, its noise covariance
$R$, and on the {\em a priori} information $P_o$, represents a measure
of the amount of information that the observable variables carry about
the unknown pose parameters. Here we explore the role of the
measurement matrix $H$ on the estimation procedure, providing the
optimal design of a sensing device able to get the maximum amount of
the information on the actual hand posture.

We first consider the {\em continuous} sensing case, where individual
sensing elements in the glove can be designed so as to measure a
linear combination of joint angles. An example of this type is the
sensorized glove developed in~\cite{Tognetti}
(cf.~figure~\ref{fig:ContinuousSensing}), or the 5DT Data Glove (5DT
Inc., Irvine, CA - USA). Other devices, such as e.g.  the Cyberglove (CyberGlove System LLC, San Jose, CA - USA), or the Humanglove (Humanware s.r.l., Pisa, Italy)
shown in figure~\ref{fig:DiscreteSensing}, provide instead {\em discrete} sensing, i.e.~each sensor provides a measure of a single joint angle.

Finally, for the sake of generality, we consider the optimal design of
hybrid sensing devices, which combine continuous and discrete
sensors. It is interesting to note that human hands represent, to some
extent, examples of such hybrid sensing: among the cutaneous
mechanoreceptors in the dorsal skin of the hand that were demonstrated
to be involved in the responses to finger movements, \cite{Edin}
includes both Fast Adapting (mainly FAI) afferents, with localized
response to movements about one or, at most, two nearby joints; and
Slow Adapting (SA) afferents, whose discharge rate is influenced by
several joints interactively. Note also that FA units are found
primarily close to joints, while SA units are more uniformly
distributed.

To validate our technique we consider hand posture reconstruction
using a limited number of measurements from a set of grasp postures
acquired with an optical tracking system, providing accurate reference
poses. Experiments and statistical analyses demonstrate the improvement of
the estimation techniques proposed in \cite{Bianchi_etalI} by using
the optimal design proposed in this paper.

\section{Problem Definition}
\label{hand}

For reader's convenience we summarize here the definitions and results of~\cite{Bianchi_etalI} used in the following. Let us assume a $n$ degrees of freedom kinematic hand model and let $y\in\real^m$ be the measures provided by a sensing glove. The relationship between joint variables $x\in\real^n$ and measurements $y$ is
\begin{equation}
	\label{eq:MeasuresGlove}
		y = Hx+\nu\,,
\end{equation}
where $H\in\real^{m\times n}$ ($m<n$) is a full row rank matrix, and $v \in \real^m$ is a vector of measurement noise. In~\cite{Bianchi_etalI}, the goal is to determine the hand posture, i.e.~the joint angles $x$, by using a set of measures $y$ whose number is lower than the number of DoFs describing the kinematic hand model in use. To improve the hand pose reconstruction, we used postural synergy information embedded in the \emph{a priori} grasp set, which is obtained by collecting a large number $N$ of grasp postures $x_i$, consisting of $n$ DoFs, into a matrix $X\in\real^{n\times N}$. This information can be summarized in a covariance matrix $P_o\in\real^{n\times n}$, which is a symmetric matrix computed as $P_o=\frac{(X-\bar{x})(X-\bar{x})^{T}}{N-1}$, where $\bar{x}$ is a matrix $n\times N$ whose columns contain the mean values for each joint angle arranged in vector $\mu_o\in\real^{n}$. 

Based on the Minimum Variance Estimation (MVE) technique, in~\cite{Bianchi_etalI} we obtained the hand pose reconstruction as
\begin{equation}
\hat{x} = (P_o^{-1}+H^{T}R^{-1}H)^{-1}(H^{T}R^{-1}y + P_o^{-1}\mu_o)\,,
\label{eq:MAP}
\end{equation}
where matrix $P_p = (P_o^{-1}+H^{T}R^{-1}H)^{-1}$ is the \emph{a posteriori} covariance matrix. When $R$ tends to assume very small values, the solution described in~\eqref{eq:MAP} might encounter numerical problems. However, by using the Sherman-Morrison-Woodbury formulae, \eqref{eq:MAP} can be rewritten as
\begin{equation}
\label{eq:MVE2}
\hat{x} = \mu_o - P_oH^{T}(HP_oH^{T}+R)^{-1}(H\mu_o-y)\,,
\end{equation}
and the {\em a posteriori} covariance matrix becomes $P_p=P_o-P_oH^T(HP_oH^T+R)^{-1}HP_o$.

The {\em a posteriori} covariance matrix, which depends on measurement matrix $H$, represents a measure of the amount of information that an observable variable carries about unknown parameters. In this paper we will explore the role of the measurement matrix $H$ on the estimation procedure, providing the optimal design of a sensing device able to obtain the maximum amount of the information on the actual hand posture.

Let us preliminary introduce some useful notations. If $M$ is a symmetric matrix with dimension $n$, let its Singular Value Decomposition (SVD) be $M=U_M\Sigma_MU_M^T$, where $\Sigma_M$ is the diagonal matrix containing the singular values $\sigma_1(M)\geq\sigma_2(M)\geq\cdots\geq\sigma_n(M)$ of $M$ and $U_M$ is an orthogonal matrix whose columns $u_i(M)$ are the eigenvectors of $M$, known as Principal Components (PCs) of $M$, associated with $\sigma_i(M)$. For example, the SVD of the \emph{a priori} covariance matrix is $P_o=U_{P_o}\Sigma_{P_o}U_{P_o}^T$, with $\sigma_i(P_o)$ and $u_i(P_o)$, $i=1,2,\dots,n$, the singular values and the principal components of matrix $P_o$, respectively.

\section{Optimal Sensing Design}
\label{sec:OptimalEstimation}

We first analyze the case that individual sensing elements in the glove can be designed to measure a linear combination of joint angles (continuous sensing devices), and provide, for given \emph{a priori} information and fixed number of measurements, the optimal design, minimizing in average the reconstruction error. We then consider the case where each measure provided by the glove corresponds to a single joint angle (discrete sensing devices). For these types of gloves we determine which joint should be individually measured in order to optimize the design. Finally, we will consider the case that both continuous and discrete sensor elements are used in the achieve sensing devices, defining a procedure to obtain the optimal hybrid sensing glove design.

In the ideal case of noiseless measures ($R=0$), $P_p$ becomes zero when $H$ is a full rank $n$ matrix, meaning that available measures contain a complete information about the hand posture. In the real case of noisy measures and/or when the number of measurements $m$ is less than the number of DoFs $n$, $P_p$ can not be zero. In these cases, the following problem becomes very interesting: find the optimal matrix $H^*$ such that the hand posture information contained in the fewer number of measurements is maximized. Without loss of generality, we assume $H$ to be full row rank and we consider the following problem.
\begin{problem}
\label{prob:OptimalH1}
Let $H$ be an $m\times n$ full row rank matrix with $m<n$ and $V_1(P_o,H,R):\real^{m\times n}\rightarrow \real$ be defined as $V_1(P_o,H,R) = \|P_o-P_oH^T(HP_oH^T+R)^{-1}HP_o\|_F^2$, find
\begin{equation*}
	\label{eq:OptimalProblem}
	\begin{split}
	& H^* = \arg\min_H\, V_1(P_o,H,R)
	\end{split}
\end{equation*}
where $\|\cdot\|_F$ denotes the Frobenius norm defined as $\|A\|_F=\sqrt{\tr(A\,A^T)}$, for $A\in\real^{n\times n}$.
\end{problem} 
To solve problem~\ref{prob:OptimalH1} means to minimize the entries of the {\em a posteriori} covariance matrix: the smaller the values of the elements in $P_p$, the greater is the predictive efficiency. 

In order to simplify the analysis, in the following we will analyze separately the design of continuous, discrete and hybrid sensing devices.

\subsection{Continuous Sensing Design}

For this case, each row of the measurement matrix $H$ is a vector in $\real^n$ and hence can be given as a linear combination of a $\real^n$ basis. Without loss of generality, we can use the principal components of matrix $P_o$, i.e. the columns of the previously defined matrix $U_{P_o}$, as a basis of $\real^n$. Consequently the measurement matrix can be written as $H=H_eU_{P_o}^T$, where $H_e\in\real^{m\times n}$ contains the coefficients of the linear combinations.  Given that $P_o=U_{P_o}\Sigma_{P_o}U_{P_o}^T$, the \emph{a posteriori} covariance matrix becomes
\begin{equation}
	\label{eq:Pp_Hc}
	P_p = U_{P_o}\left[\Sigma_o-\Sigma_oH_e^T(H_e\Sigma_oH_e^T+R)^{-1}H_e\Sigma_o\right]U_{P_o}^T\,,
\end{equation}
where, for simplicity of notation $\Sigma_o\equiv \Sigma_{P_o}$. 

Next sections are dedicated to describe the optimal continuous sensing design both in a numerical and analytical way. For this purpose, let us introduce the set of $m\times n$ (with $m<n$) matrices with orthogonal rows, i.e.~satisfying the condition $HH^T=I_{m\times m}$, and let us denote it as $\mathcal{O}_{m\times n}$.

\subsubsection{Analytical Solutions}

We first consider the case of noiseless measures, i.e.~$R=0$. Let $A$ be a non-negative matrix of order $n$. It is well known (cf.~\cite{Rao64}) that, for any given matrix ${B}$ of rank $m$ with $m\leq n$,
\begin{equation}
	\label{eq:OptimalB}
	\min_{B}\|A-{B}\|_F^2=\alpha_{m+1}^2+\cdots+\alpha_n^2\,,
\end{equation}
where $\alpha_i$ are the eigenvalues of $A$, and the minimum is attained when
\begin{equation}
	\label{eq:OptimalChoiceB}
	{B}=\alpha_1{w_1w_1}^T+\cdots+\alpha_m{w_mw_m}^T\,,
\end{equation}	
where $w_i$ are the eigenvector of $A$ associated with $\alpha_i$.
In other words, the choice of ${B}$ as in~\eqref{eq:OptimalChoiceB} is the best fitting matrix of given rank $m$ to $A$. By using this result we are able to show when the minimum of~\eqref{eq:Pp_Hc}, hence of 
\begin{equation}
	\label{eq:CostHc}
\|\Sigma_o-\Sigma_oH_e^T(H_e\Sigma_oH_e^T)^{-1}H_e\Sigma_o\|_F^2\,,
\end{equation}
can be reached. Let us preliminary observe that the row vectors $(h_i)_e$ of $H_e$ can be chosen, without loss of generality, to satisfy the condition $(h_i)_e\,\Sigma_o\,(h_j)_e=0,\ i\neq j$, which implies that the measures are uncorrelated (\cite{Rao64}). Let $\mathcal{O}_{m\times n}$ denotes the set of $m\times n$ matrices, with $m<n$, whose rows satisfy the aforementioned condition, i.e.~the set of matrices with orthonormal rows ($H_eH_e^T=I$).  By using \eqref{eq:OptimalB}, the minimum of~\eqref{eq:CostHc} is obtained when (cf.~\cite{Rao64})
\begin{equation}
	\label{eq:OptimalHc}
\begin{aligned}
\Sigma_oH_e^T(H_e\Sigma_oH_e^T)^{-1}H_e\Sigma_o &= \sigma_1(\Sigma_o)u_1(\Sigma_o)u_1^T(\Sigma_o)+\cdots+\\
&+\sigma_m(\Sigma_o)u_m(\Sigma_o)u_m^T(\Sigma_o)\,.
\end{aligned}
\end{equation}
Since $\Sigma_o$ is a diagonal matrix, $u_i(\Sigma_o)\equiv e_i$,
where $e_i$ is the $i$\emph{-th} element of the canonical
basis. Hence, it is easy to verify that~\eqref{eq:OptimalHc} holds for $H_e=[I_{m}\, |\, 0_{m\times(n-m)}]$. As a consequence, row vectors $(h_i)$ of $H$ are the first $m$ principal components of $P_o$, i.e.~$(h_i)_=u_i(P_o)^T$, for $i=1,\dots,m$.

From these results, a principal component can be defined as a linear combination of optimally-weighted observed variables meaning that the corresponding measures can account for a maximal amount of variance in the data set. As reported in~\cite{Rao64}, every set of $m$ optimal measures can be considered as a representation of points in the best fitting lower dimensional subspace. Thus the first measure gives the
best one--dimensional representation of data set, the first two measures give the best two--dimensional representation, and so on.

In the noisy measurement case, \eqref{eq:OptimalHc} can be rewritten as  
\[
\begin{aligned}
\Sigma_oH_e^T(H_e\Sigma_oH_e^T+R)^{-1}H_e\Sigma_o &- \sigma_1(\Sigma_o)u_1(\Sigma_o)u_1^T(\Sigma_o)+\cdots+\\
&+\sigma_m(\Sigma_o)u_m(\Sigma_o)u_m^T(\Sigma_o) = \Delta
\end{aligned}
\]
In this case, $\Delta = 0$ can not be attained for any finite $H$: indeed, for unconstrained $H$, $\inf_{H} V_1(P_0,H,R)$ would be attained for $\|H\| \rightarrow \infty$, i.e.~for infinite signal-to-noise ratio. The problem can be recast in a well--posed form by imposing a constraint on the magnitude of the measurement matrix. Up to a possible renormalization of $R$, we can search the optimum design in the set $\mathcal{A}=\{H\,:\,HH^T=I_m\}$. This problem was discussed and solved in~\cite{Kostas93}, showing that, for arbitrary noise covariance matrix $R$,
\begin{equation}
	\label{eq:MinFuncV1_Noise}
	\min_{H\in\mathcal{A}}V_1(H)=\sum_{i=1}^m\frac{\sigma_i(P_o)}{1+\sigma_i(P_o)/\sigma_{m-i+1}(R)}+\sum_{i=m+1}^n\sigma_i(P_o)\,,
\end{equation}
which is attained for
\begin{equation}
	\label{eq:OptimalH_Noise}
	H=\sum_{i=1}^mu_{m-i+1}(R)u_i(P_o)\,.
\end{equation}
Hence, if $\mathcal{A}$ consists of all matrices with mutually
perpendicular, unit length rows, the first $m$ principal components of $P_o$ are still the optimal choice for $H$ rows. The alternative case that the solution is sought under a Frobenius norm constraint on $H$, i.e.~$\mathcal{A}=\{H:\|H\|_F\leq 1\}$ is discussed in~\cite{Kostas93}.

\subsubsection{Numerical Solution: Gradient flows on $\mathcal{O}_{m\times n}$}

In this subsection we describe a different approach to the solution of problem~\ref{eq:OptimalProblem}, which consists of constructing a differential equation whose trajectories converge to the desired optimum. The method lends itself directly to efficient numerical implementations. Although a closed-form solution has been proposed in the previous subsection, the numerical solution considered here is very useful when constraints are imposed on the measurement structure (as they will be for instance in the hybrid sensor design), where closed form solutions are not applicable.
 
The following proposition describes an algorithm that minimizes the cost function $V_1(P_o,H,R)$, providing the gradient flow which will be useful in the method of steepest descent.
\begin{proposition}
\label{prop:GradientFlow}
The gradient flow for the function $V_1(P_o,H,R):\real^{m\times n}\rightarrow \real$ is given by,
\begin{equation}
	\label{eq:GradientFlow}
	\dot H = -\nabla \|P_p\|_F^2= 4\left[P_p^2P_oH^T\Sigma(H)\right]^T\,,
\end{equation}
where $\Sigma(H) = (HP_oH^T+R)^{-1}$.
\end{proposition}
\begin{proof} See Appendix.
\end{proof}

Let us observe that rows of matrix $H$ can be chosen, without loss of generality, such that $H_iP_oH_j^T=0,\ i\neq j$ which imply that measures are uncorrelated, i.e.~satisfying the condition $HH^T=I_m$. Of course, in case of noise--free sensors, this constraint is not strictly necessary. On the other hand, in case of noisy sensors, the minimum of $V_1(P_o,H,R)$ can not be obtained since it represents a limit case that can be achieved when $H$ becomes very large (i.e.~an infimum) and hence increasing the signal-to-noise ratio.

A reasonable solution for the constrained problem will be provided by using the Rosen's gradient projection method for linear constraints~\cite{Rosen}, which is based on projecting the search direction into the subspace tangent to the constraint. Hence, given the steepest descent direction for the unconstrained problem, this method consists on finding the direction with the most negative directional derivative which satisfies the constraint on the structure of the matrix $H$, i.e.~$HH^T=I_{m}$. This can be obtained by using the  projection matrix
\begin{equation}
\label{eq:Proj}
W = I_{m} - H(H^TH)^{-1}H^T\,,
\end{equation}
and then projecting the unconstrained gradient flow~\eqref{eq:GradientFlow} into the subspace tangent to the constraint, obtaining the search direction
\begin{equation}
	\label{eq:SearchDir}
	s = W\ \nabla \|P_p\|_F^2\,.
\end{equation}

Having the search direction for the constrained problem, the gradient flow is given by
\begin{equation}
	\label{eq:ConstrGradientFlow}
	\dot H = -4W\ \left[P_p^2P_oH^T\Sigma(H)\right]^T
\end{equation}
where $\Sigma(H) = (HP_oH^T+R)^{-1}$. The gradient flow~\eqref{eq:GradientFlow} guarantees that the optimal solution $H^*$ will satisfy $H^*(H^*)^T=I_m$, if $H(0)$ satisfies $H(0)H(0)^T=I_m$, i.e.~$H\in\mathcal{O}_{m\times n}$.

Notice that both $\mathcal{O}_{m\times n}$ and $V_1(P_o,H,R)$ are not convex, hence the problem could not have a unique minimum. However, in case of noise--free measures, the invariance of the cost function w.r.t.~changes of basis, i.e.~$V_1(P_o,H,0)=V_1(P_o,MH,0)$ with $M\in\real^m$ a full rank matrix, suggests that there exists a subspace in $\real^n$ where the optimum is achieved. Indeed, gradients become zero when rows of matrix $H$ are any linear combination of a subset of $m$ principal components of the {\em a priori} covariance matrix. Unfortunately, this does not happen in case of noisy measures and gradients become zero only for a particular matrix $H$ which depends also on the principal components of the noise covariance matrix.

\subsection{Discrete Sensing Design}
 
When each measure $y_j$, $j=1,\dots,m$ provided by the glove corresponds to a single joint angle $x_i$, $i=1,\dots,n$, the problem is to find the optimal choice of $m$ joints or DoFs to be measured. 

Measurement matrix becomes in this case a full row rank matrix where each row is a vector of the canonical basis, i.e.~matrices which have exactly one nonzero entry in each row.

Let $\mathcal{N}_{m\times n}$ denote the set of $m\times n$ element-wise non-negative matrices, then $\mathcal{P}_{m\times n}=\mathcal{O}_{m\times n}\cap\mathcal{N}_{m\times n}$, where $\mathcal{P}_{m\times n}$ is the set of $m\times n$ permutation matrices (see lemma~2.5 in \cite{Pappas}). This result implies that if we restrict $H$ to be orthonormal and element-wise non-negative, we get a permutation matrix. In this paper we extend this result in $\real^{m\times n}$, obtaining matrices which have exactly one nonzero entry in each row. Hence, the problem to solve becomes:

\begin{problem}
\label{prob:OptimalH2}
Let $H$ be a $m\times n$ matrix with $m<n$, and $V_1(P_o,H,R):\real^{m\times n}\rightarrow \real$ be defined as $V_1(P_o,H,R) = \|P_o-P_oH^T(HP_oH^T+R)^{-1}HP_o\|_F^2$, find the optimal measurement matrix
\begin{equation*}
	\label{eq:OptimalProblem2}
	\begin{split}
	& H^* = \arg\min_{H}\, V_1(P_o,H,R)\\
	& s.t.\quad H\in\mathcal{P}_{m\times n}\,.
	\end{split}
\end{equation*}
\end{problem}

In this case a closed-form solution is not available. Nonetheless, as the model hand adopted has usually a low number of DoFs, the optimal choice $H^*$ can be computed by exhaustion, substituting all possible sub--sets of $m$ vectors of the canonical basis in the cost function $V_1(P_o,H,R)$. In next section, a more general approach to computing the optimal matrix will be provided in order to obtain a result also when a model with a large number of DoFs is considered.

\subsubsection{Numerical Solution: Gradient Flows on $\mathcal{P}_{m\times n}$} 

In this section, we describe an alternative approach to the solution of problem~\ref{eq:OptimalProblem2} based on a gradiental method. Once again, although the enumeration approach can solve the problem in practical cases, the numerical solution based on the method here presented will be useful in the design of hybrid sensors.

A numerical solution for problem~\ref{prob:OptimalH2} can be obtained following a method presented in~\cite{Pappas}, which consists in defining a function $V_2(P)$ with $P\in\real^{n\times n}$ that forces the entries of $P$ to be as positive as possible, thus penalizing negative entries of $H$. In this paper, we extend this function to measurement matrices $H\in\real^{m\times n}$ with $m<n$. Consider a
function $V_2:\mathcal{O}_{m\times n}\rightarrow \real$ as
\begin{equation}
\label{eq:PermutationIndex}
V_2(H) = \frac{2}{3}\tr\left[H^T(H-(H\circ H))\right]\,,
\end{equation}
where $A\circ B$ denotes the \emph{Hadamard} or elementwise product of the matrices $A=(a_{ij})$ and $B=(b_{ij})$, i.e. $A\circ B=(a_{ij}b_{ij})$. The gradient flow of $V_2(H)$ is given by (\cite{Pappas})
\begin{equation}
\label{eq:GradPermutationIndex}
\dot H = -H\left[(H\circ H)^TH-H^T(H\circ H)\right]\,,
\end{equation}
which minimizes $V_2(H)$ converging to a permutation matrix if $H(0)\in\mathcal{O}_{m\times n}$.

The two gradient flows given by~\eqref{eq:GradientFlow}
and~\eqref{eq:GradPermutationIndex}, both defined on the space of orthogonal matrices, tend to respectively minimize their cost functions. By combining these two gradient flows we can achieve a solution for Problem~\ref{prob:OptimalH2}. An interesting result applies to the dynamics of the convex combination of these gradients, which can be stated as follows.

\begin{theorem}
\label{thm:FinalGradientFlow}
Let $H\in\real^{m\times n}$ with $m<n$ be the measurement process matrix and assume that $H(0)\in\mathcal{O}_{m\times n}$. Moreover, suppose that $H(t)$ satisfies the following matrix differential equation,
\begin{align}
	\label{eq:FinalGradientFlow}
        \dot H &= 4\,(1-k)W\ \left[P_p^2P_oH^T\Sigma(H)\right]^T+\nonumber\\
        &+k\ H\left[(H\circ H)^TH-H^T(H\circ H)\right]\,,
\end{align}
where $k\in[0,\,1]$ is a positive constant and $\Sigma(H) = (HP_oH^T+R)^{-1}$. For sufficiently large $k$, $\lim_{t\rightarrow \infty} H(t)=H_\infty$ exists and approximates a permutation matrix that also (locally) minimizes the squared Frobenius norm of the
\emph{a posteriori} covariance matrix, $\|P_p\|_F^2$.
\end{theorem}

The proof of this theorem is a direct extension of results
in~\cite{Pappas}, and is omitted for brevity.  

As in most numerical optimization algorithms, the non-convex nature of the cost function and of the support set implies the need for multi-start approaches. A possible technique to help converge towards the global optimum consists in increasing $k$ during the search procedure (cf.~\cite{Pappas}).

\subsection{Hybrid Sensing Design}

In this section we analyze the sensing device with both continuous and discrete sensors. Up to rearranging the sensor numbering, we can write a hybrid measurement matrix
$H_{c,d}\in\real^{m\times n}$ as
\[
H_{c,d}=
\left[
\begin{array}{c}
	H_c\\
	\cline{1-1}
	H_d
\end{array}
\right]\,,
\]
where $H_c\in\real^{m_c\times n}$ defines the $m_c$ continuous sensing elements, whereas $H_d\in{\mathcal{P}}^{m_d\times n}$ describes the $m_d$ single-joint measurements, with $m_c+m_d=m$. Neither the closed-form solution valid for continuous sensing design, nor the exhaustion method used for discrete measurements are applicable in the hybrid case. Therefore, to optimally design hybrid pose sensing systems, we will recur to gradient-based iterative optimization algorithms. 

We first consider the case that noise is negligible ($R\approx 0$). By combining the continuous and discrete gradient flows, previously defined in \eqref{eq:GradientFlow} and
\eqref{eq:GradPermutationIndex}, respectively, we obtain
\begin{align}
\label{eq:HybridGradientFlow}
\dot H_{c,d} &= 4\,(1-k)\ \left[P_p^2P_oH_{c,d}^T\Sigma(H_{c,d})\right]^T+\nonumber\\
&+k\ \bar H_d\left[(\bar H_d\circ \bar H_d)^T\bar H_d-\bar H_d^T(\bar H_d\circ \bar H_d)\right]\,,
\end{align}
where $k\in[0,\,1]$ is a positive constant, $\Sigma(H_{c,d}) =
(H_{c,d}P_oH_{c,d}^T)^{-1}$, and
\[
\bar H_d= \left[
	\begin{array}{c}
		0_{m_c\times n}\\
		\cline{1-1}
		H_d
	\end{array}
	\right]\,.
\]

On the basis of Theorem~\ref{thm:FinalGradientFlow}, the gradient flow defined in~\eqref{eq:HybridGradientFlow} converges toward a hybrid sensing system (locally) minimizing the squared Frobenius norm of the {\em a posteriori} covariance matrix. Multi--start strategies have to be used to circumvent the problem of local minima.\\

When noise is not negligible, the gradient search method
of~\eqref{eq:HybridGradientFlow} would tend to produce measurement matrices whose continuous parts, $H_{c}$, are very large in norm. This is an obvious consequence of the fact that, for a fixed noise covariance $R$, larger measurement matrices $H$ would produce an apparently higher signal-to-noise ratio in~\eqref{eq:MeasuresGlove}. 

This problem can be circumvented by constraining the solution in the sub-set $\mathcal{H}_{c,d}=\{H_{c,d}:H_{c} H_{c}^T = I_{m_c}\}$. A solution for this problem can be obtain by the following gradient flow
\begin{align}
\label{eq:HybridGradientFlowN}
\dot H_{c,d} &= 4\,(1-k)\,W_{c,d}\ \left[P_p^2P_oH_{c,d}^T\Sigma(H_{c,d})\right]^T+\nonumber\\
&+k\ \bar H_d\left[(\bar H_d\circ \bar H_d)^T\bar H_d-\bar H_d^T(\bar H_d\circ \bar H_d)\right]\,,
\end{align}
where $k\in[0,\,1]$ is a positive constant,
$P_p=P_o-P_oH_{c,d}^T(H_{c,d}P_oH_{c,d}^T+R)^{-1}H_{c,d}P_o$, and
$\Sigma(H_{c,d})=(H_{c,d}P_oH_{c,d}^T +R)^{-1}$. With the choice 
\[
W_{c,d}=
\begin{bmatrix}
I_{m_c}-H_c(H_c^TH_c)^{-1}H_c^T & 0_{m_c\times m_d}\\
0_{m_d\times m_c} & I_{m_d\times m_d}
\end{bmatrix}
\]
for the projection matrix, and starting from any initial guess matrix $H_{c,d}\in\mathcal{H}_{c,d}$, the gradient flow remains in the sub-set $\mathcal{H}_{c,d}$, and converges to a (local) minimum for the problem. Also in this case, multi–start strategies can circumvent the problem of local minima.

\section{Results}
\label{OptRis}

In this section we will describe how the information available by measurement process increases with the minimization of the squared Frobenius norm of the {\em a posteriori} covariance matrix as well as increasing the number of measures, leading to better estimation performance.

First, based on the {\em a priori} covariance matrix obtained with the {\em a priori} data set described in section 3 of~\cite{Bianchi_etalI}, we will show the optimal distribution of sensors on the hand in case of continuous and discrete sensing devices. We will also show that, although the number of measures used with the optimal matrix is less than the five measures available by matrix $H_s$ (cf.~\cite{Bianchi_etalI}), the hand posture information achievable with the optimal measurement matrix $H^*_d$ related to a discrete sensing device, is greater, i.e.~$V_1(H^*_d)<V_1(H_s)$, leading to a better hand pose estimation performance.

Second, we will compare the hand posture reconstruction obtained by means of matrix $H_s$ with the one obtained by using the optimal matrix $H^*_d$ with the same number of measures. Additional random normal noise $\nu$ with standard deviation of $7^{\circ}$ on each measure is also considered to evaluate the performance in case of noisy measures.

\subsection{Continuous, Discrete and Hybrid Sensing Distribution}

As shown in section~\ref{sec:OptimalEstimation}, in case of continuous sensing design, the optimal choice $H^*_c$ of the measurement matrix $H\in\real^{m\times n}$ is represented by the first $m$ principal components (synergies) of the {\em a priori} covariance matrix $P_o$. Figure~\ref{fig:ContinuousSensorDistribution} shows the hand sensor distribution related to each synergy.

\begin{figure}[t!]
\begin{center}
\includegraphics[width=0.9\textwidth]{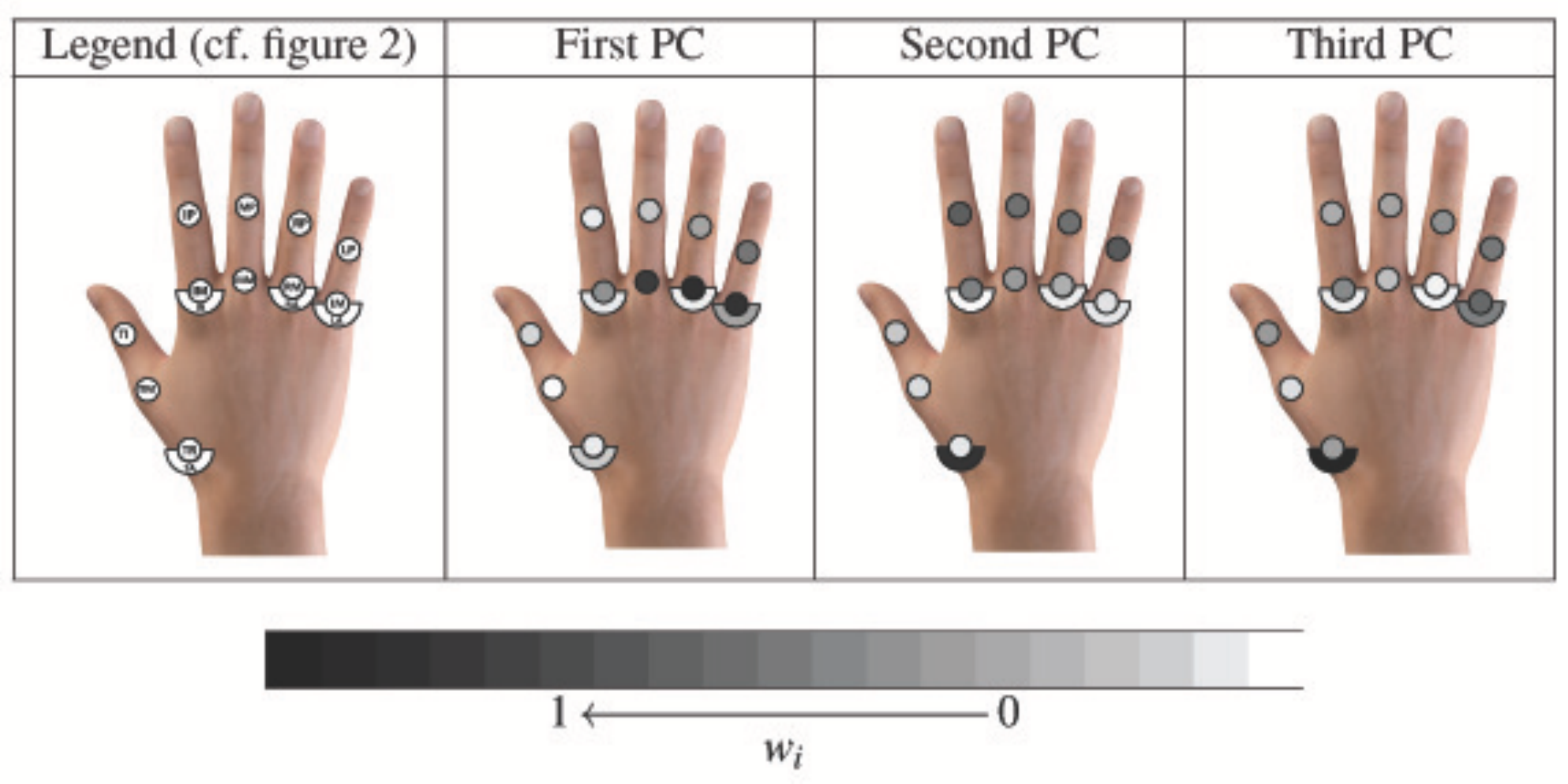}
\end{center}
\caption{Optimal continuous sensing distribution for the first PCs of $P_o$. The greater is the absolute coefficient $w_i$ of the joint angle in the PC, the darker is the color of that joint. We assume the coefficient of the $i$-{\em th} joint in the PC to be normalized w.r.t.~the maximum absolute value of the coefficients that can be achieved all over the joints.}
\label{fig:ContinuousSensorDistribution}
\end{figure}

In case of discrete sensing, the optimal measurement matrix $H^*_d$, related to a discrete sensing device, for a number of noise--free measures $m$ ranging from 1 to 14, is reported in table~\ref{tab:OptimalSelMatrix}. Notice that, $H^*_d$ does not have an incremental behaviour, especially in case of few measures. In other words, the set of DoFs which have to be chosen in case of $m$ measures does not necessarily contain all the set of DoFs chosen for $m-1$ measures. Moreover, noise randomness can slightly change which DoFs have to be measured compared with the noise--free case.

\begin{table}[t!]
	\renewcommand{\arraystretch}{1}
	\centering
	\small{
\begin{tabular}{|@{ }c@{ }||@{ }c@{ }|@{ }c@{ }|@{ }c@{ }|@{ }c@{ }|@{ }c@{ }|@{ }c@{ }|@{ }c@{ }|@{ }c@{ }|@{ }c@{ }|@{ }c@{ }|@{ }c@{ }|@{ }c@{ }|@{ }c@{ }|@{ }c@{ }|@{ }c@{ }||@{ }c@{ }|}
	\hline
m & TA & TR & TM & TI & IA & IM & IP & MM & MP & RA & RM & RP & LA & LM & LP & $V_1$\\
\hline
\hline
1 & & & & & & & & & & & X & & & & & $7.12\cdot10^{-2}$\\
\hline
2 & & & & & & & X & & & & X & & & & & $2.39\cdot10^{-2}$\\
\hline
3 & X & & & & & & & & & & X & X & & & & $6.59\cdot10^{-3}$\\
\hline
4 & X & & & & & & & X & & & & X & & X & & $3.30\cdot10^{-3}$\\
\hline
5 & X & & & & & & & X & & & & X & X & X & & $1.90\cdot10^{-3}$\\
\hline
6 & X & & & X & & & & X & & & & X & X & X & & $5.32\cdot10^{-4}$\\
\hline
7 & X & & & X & & & & X & X & & & & X & X & X & $2.92\cdot10^{-4}$\\
\hline
8 & X & & & X & X & & & X & X & & & & X & X & X & $1.98\cdot10^{-4}$\\
\hline
9 & X & & & X & X & & X & X & & & & X & X & X & X & $1.30\cdot10^{-4}$\\
\hline
10 & X & & & X & X & X & X & X & & & & X & X & X & X & $6.86\cdot10^{-5}$\\
\hline
11 & X & X & & X & X & X & X & X & & & & X & X & X & X & $2.70\cdot10^{-5}$\\
\hline
12 & X & X & & X & X & X & X & X & X & & & X & X & X & X & $1.40\cdot10^{-5}$\\
\hline
13 & X & X & X & X & X & X & X & X & X & & & X & X & X & X & $3.39\cdot10^{-6}$\\
\hline
14 & X & X & X & X & X & X & X & X & X & X & & X & X & X & X & $1.32\cdot10^{-6}$\\
\hline
\end{tabular}}
\caption{Optimal measured DoFs for $H^*_d$ with increasing number of noise--free measures $m$ (cf.~figure~\ref{fig:KinMod}).}
\label{tab:OptimalSelMatrix}
\end{table}

\begin{figure}[!t]
	\centering
	\begin{tabular}[c]{c}
		\includegraphics[width=0.4\columnwidth]{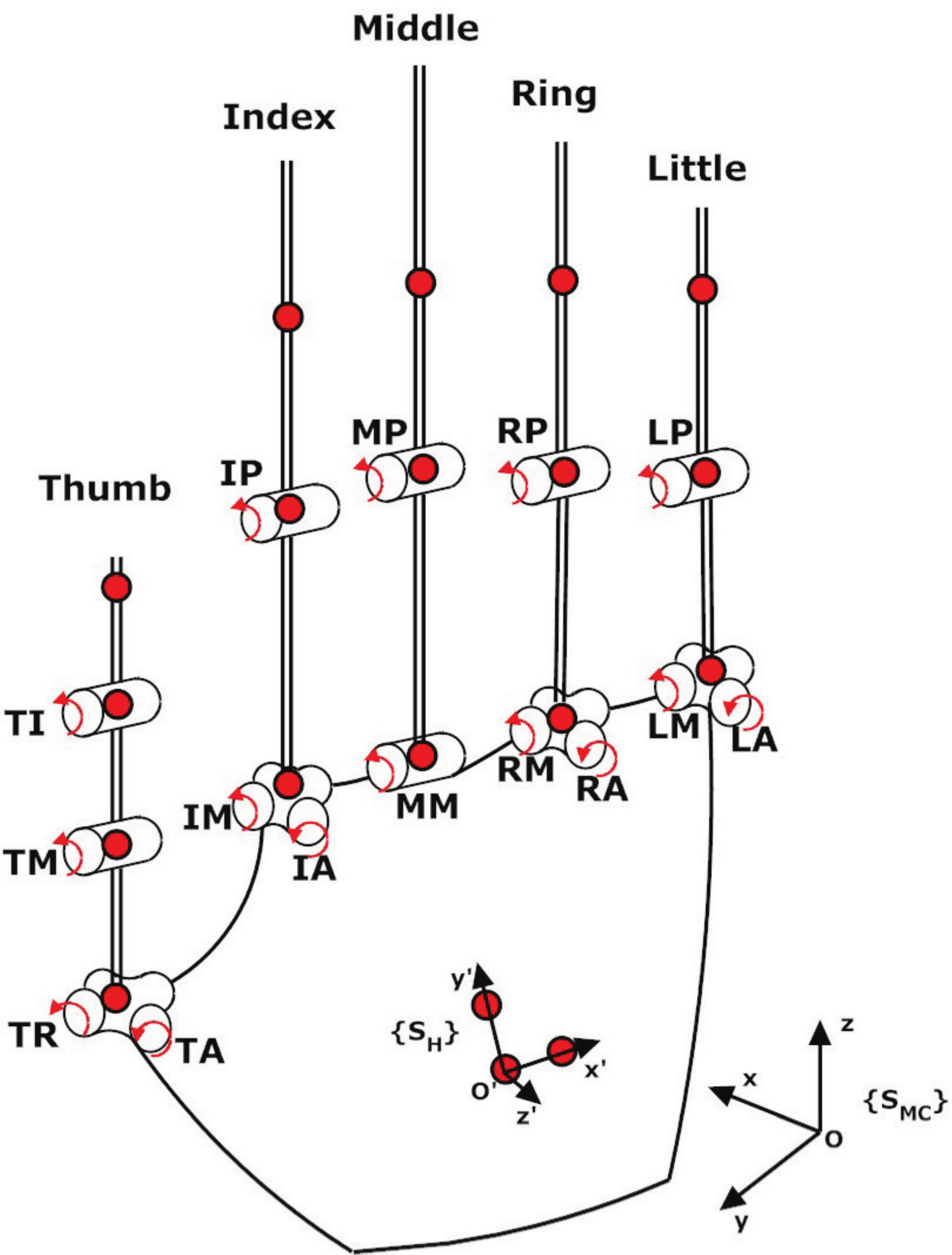}
	\end{tabular}
	\hspace{-.4cm}
	\renewcommand{\arraystretch}{1.05}
	\tiny{
	\begin{tabular}[c]{|c|c|}
		\hline
		\textbf{DoFs} & \textbf{Description}\\
		\hline
		TA & Thumb Abduction\\
		\hline
		TR & Thumb Rotation\\
		\hline
		TM & Thumb Metacarpal\\
		\hline
		TI & Thumb Interphalangeal\\
		\hline
		IA & Index Abduction\\
		\hline
		IM & Index Metacarpal\\
		\hline
		IP & Index Proximal\\
		\hline
		MM & Middle Metacarpal\\
		\hline
		MP & Middle Proximal\\
		\hline
		RA & Ring Abduction\\
		\hline
		RM & Ring Metacarpal\\
		\hline
		RP & Ring Proximal\\
		\hline
		LA & Little abduction\\
		\hline
		LM & Little Metacarpal\\
		\hline
		LP & Little Proximal\\
		\hline
	\end{tabular}}
	\caption{Kinematic model of the hand with 15 DoFs. Markers are reported as red spheres.}
\label{fig:KinMod}
\end{figure}

Figure~\ref{fig:Indtot} shows the values of the square Frobenius norm of the {\em a posteriori} covariance matrix for increasing number $m$ of noise--free measures.
The best performance is obtained by the continuous sensing design, as aspected. Indeed, principal components are considered the optimal measures for the representation of points in the best fitting lower dimensional subspace~\cite{Rao64}.
The hybrid performance is better than the discrete one, thus representing a trade-off between the quality of estimation of the continuous sensing design and feasibility and costs of the discrete one. Moreover, $V_1$ values decrease with the number of measures, tending to be zero (cf.~figure~\ref{fig:Indtot}). This fact is trivial because increasing the measurements, the uncertainty on the measured variables is reduced. When all the measured information is available $V_1$ assumes zero value with perfectly accurate measures. In case of noisy measures, $V_1$ values decrease with the number of measures tending to a value which is larger, depending on the level of noise.
\setcounter{figure}{3}
\begin{figure}[t!]
\centering
\includegraphics[width=1\columnwidth]{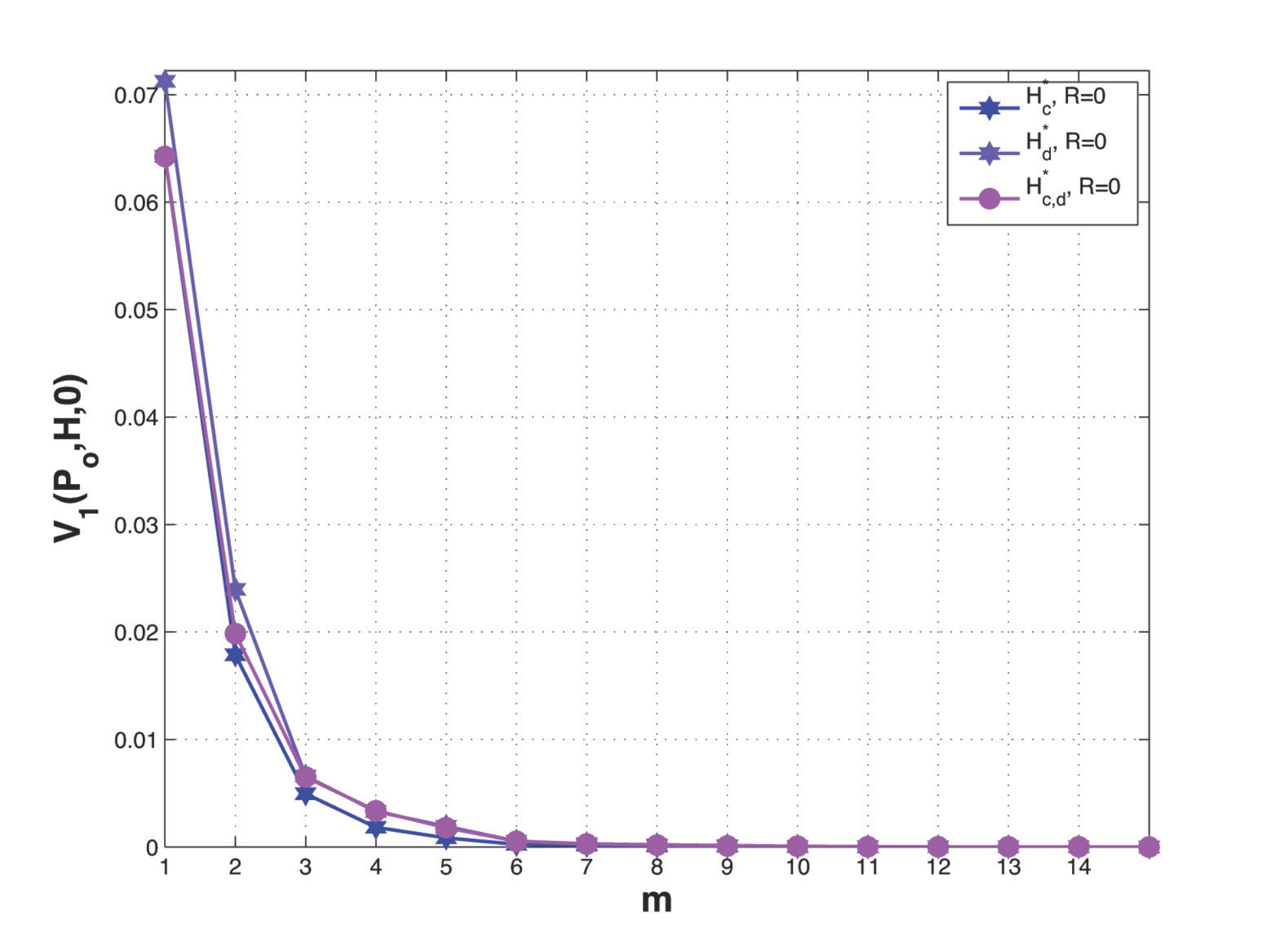}
\caption{Squared Frobenius norm of the \emph{a posteriori} matrix with noise--free measures in case of $H^*_{c}$, $H^*_{d}$ and $H^*_{c,d}$ ($m_c=1$).}
\label{fig:Indtot}
\end{figure}

For noise--free measures, if we analyze how much $V_1$ reduces with the number of measurements w.r.t.~the value it assumes for one measure, reduction percentage with three measured DoFs is greater than 80\%.
This result suggests that with only three measurements, the optimal matrix can furnish more than 80\% of uncertainty reduction. This is equivalent to say that a reduced number of measurements is sufficient to guarantee a good hand posture estimation. In~\cite{Santelloart} and \cite{Gabicciniart}, under the \emph{controllability} point of view, authors state that three postural synergies are crucial in grasp pre-shaping as well as in grasping force optimization since they take into account for more than 80\% of variance in grasp poses. Here, the same result can be obtained in terms of measurement process, i.e.~from the {\em observability} point of view: a reduced number of measures coinciding with the first three principal components enable for more than 80\% reduction of the squared Frobenius norm of the \emph{a posteriori} covariance matrix.

The above reported result seems logic considering the duality between observability and controllability. Moreover, under an engineering point of view, it is reasonable that those actuators which are used the most being also the most monitored and hence the most sensor endowed. 

\section{Discussion}

In this section, we will compare the hand posture reconstruction obtained by applying the hand pose reconstruction techniques described in~\cite{Bianchi_etalI} to $m=5$ measures provided by matrix $H_s$ and by optimal matrix $H^*_d$.

\setcounter{figure}{4}
\begin{figure*}[t!]
\begin{center}
	\renewcommand{\arraystretch}{1}
\begin{tabular}[c]{|@{}c@{}|c|c|}
	\hline
\small{ Legend (cf.~figure~\ref{fig:KinMod}) } & $H_s$ & $H^*_d$\\
\hline
\includegraphics[width=0.2\textwidth]{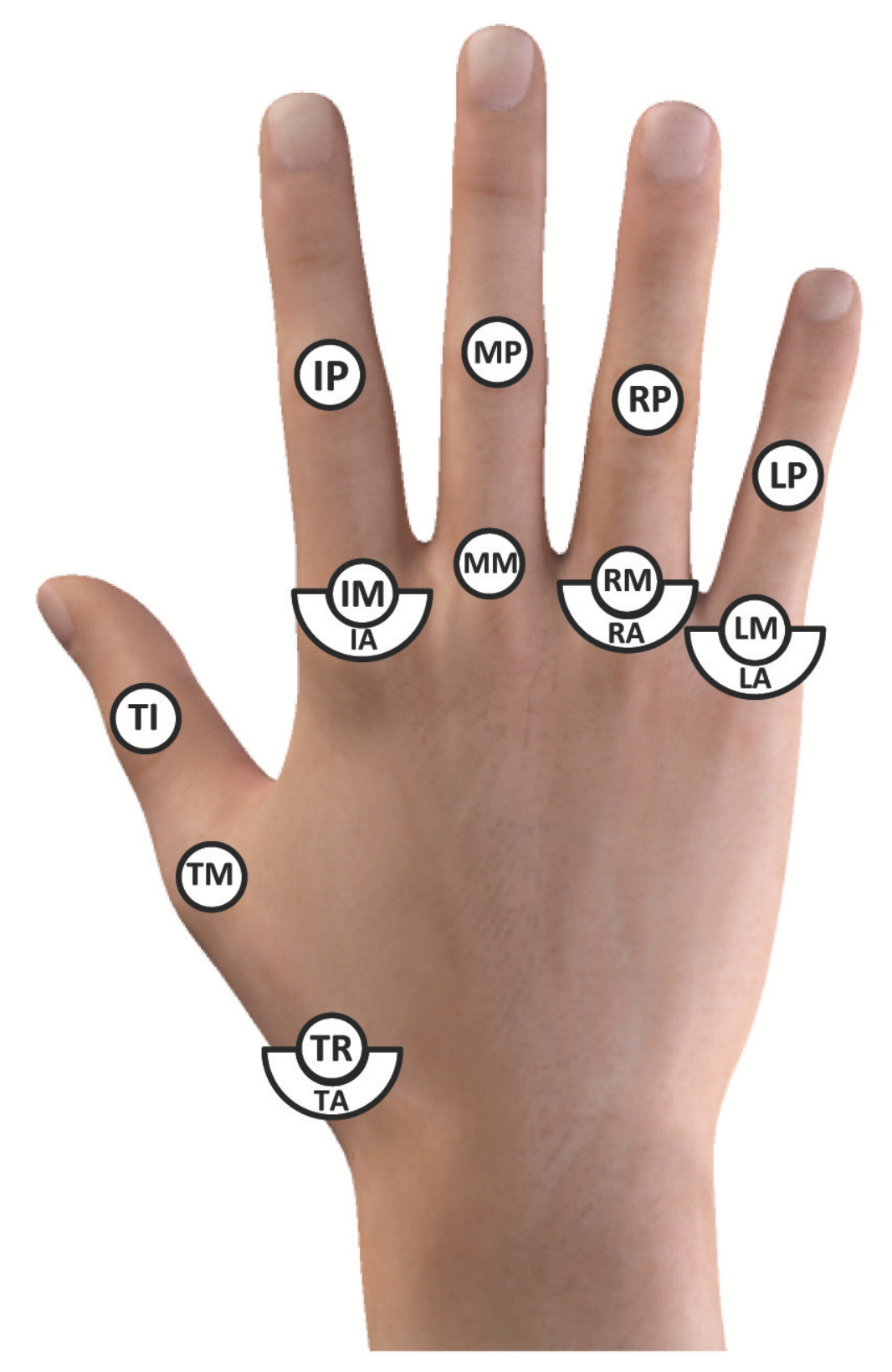} &
\includegraphics[width=0.2\textwidth]{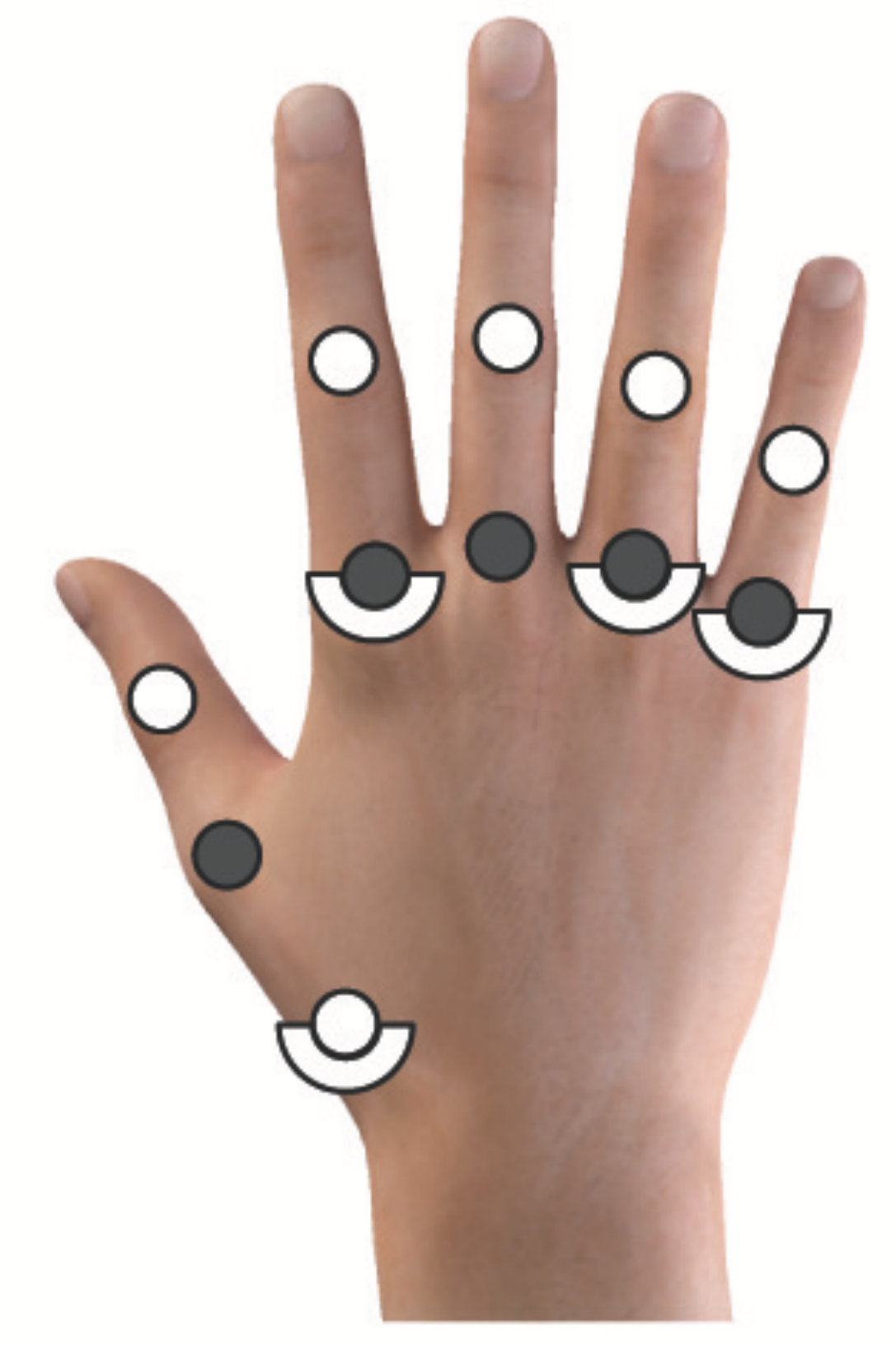} &
\includegraphics[width=0.2\textwidth]{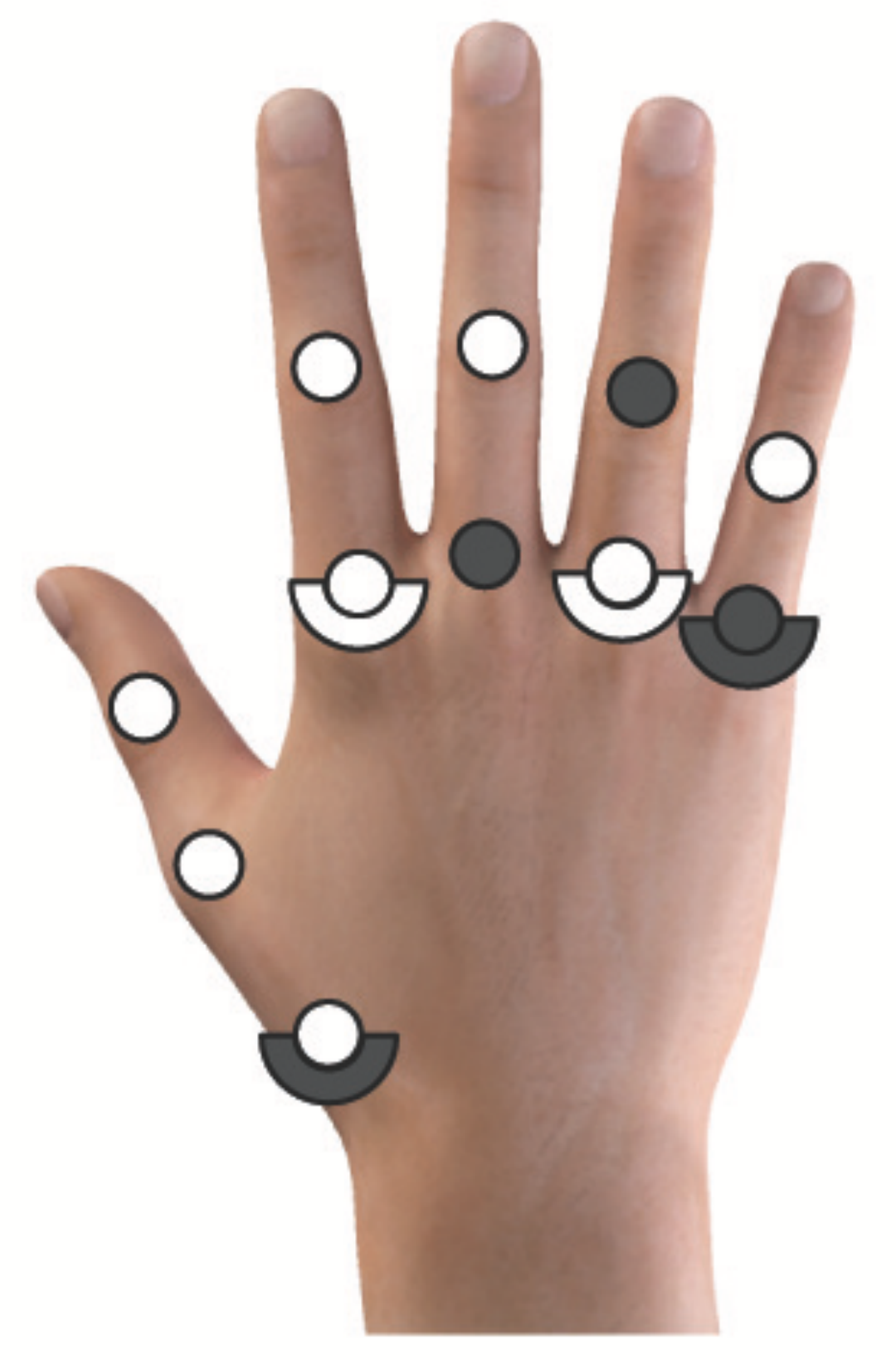}\\
\hline
\cline{2-3}
\multicolumn{1}{r|}{ } & \Tiny{Measured joints: $TM$, $IM$, $MM$, $RM$ and $LM$} & \Tiny{Measured joints: $TA$, $MM$, $RP$, $LA$ and $LM$}\\
\cline{2-3}
\end{tabular}
\end{center}
\caption{Discrete sensing distributions for matrix $H_s$, on the left, and $H^*_d$, on the right (cf.~figure~\ref{fig:KinMod}). The measured joints are highlighted in color.}
\label{fig:Hs-HOptSel}
\end{figure*}

\subsection{Estimation Results with Optimal Discrete Sensing Devices}

Measures are provided by grasp data acquired with the optical tracking system as in~\cite{Bianchi_etalI}, where degrees of freedom to be measured are chosen on the basis of optimization procedure outcomes, while the entire pose is recorded to produce accurate reference posture. In figure~\ref{fig:Hs-HOptSel} sensor locations related to matrix $H_s$ and  $H^*_d$ are represented. In order to compare reconstruction performance achieved with $H_s$ and $H^*_d$ we use as evaluation indices the average pose estimation error and average estimation error for each estimated DoF. Maximum errors are also reported. These errors as well as statistical tools are chosen according to the ones considered in~\cite{Bianchi_etalI}, where it is possible to find a complete description of the here adopted. Both noise-free and noisy measures are analyzed.

\setcounter{figure}{5}
\begin{figure*}[t!]
\begin{center}
	\renewcommand{\arraystretch}{1}
\begin{tabular}[c]{c}
Real Hand Postures\\
\end{tabular}
\begin{tabular}[c]{p{0.3cm}|p{2.4cm}|p{2.4cm}|p{2.4cm}|p{2.4cm}|}
\cline{2-5}
 &
\includegraphics[width=0.21\textwidth]{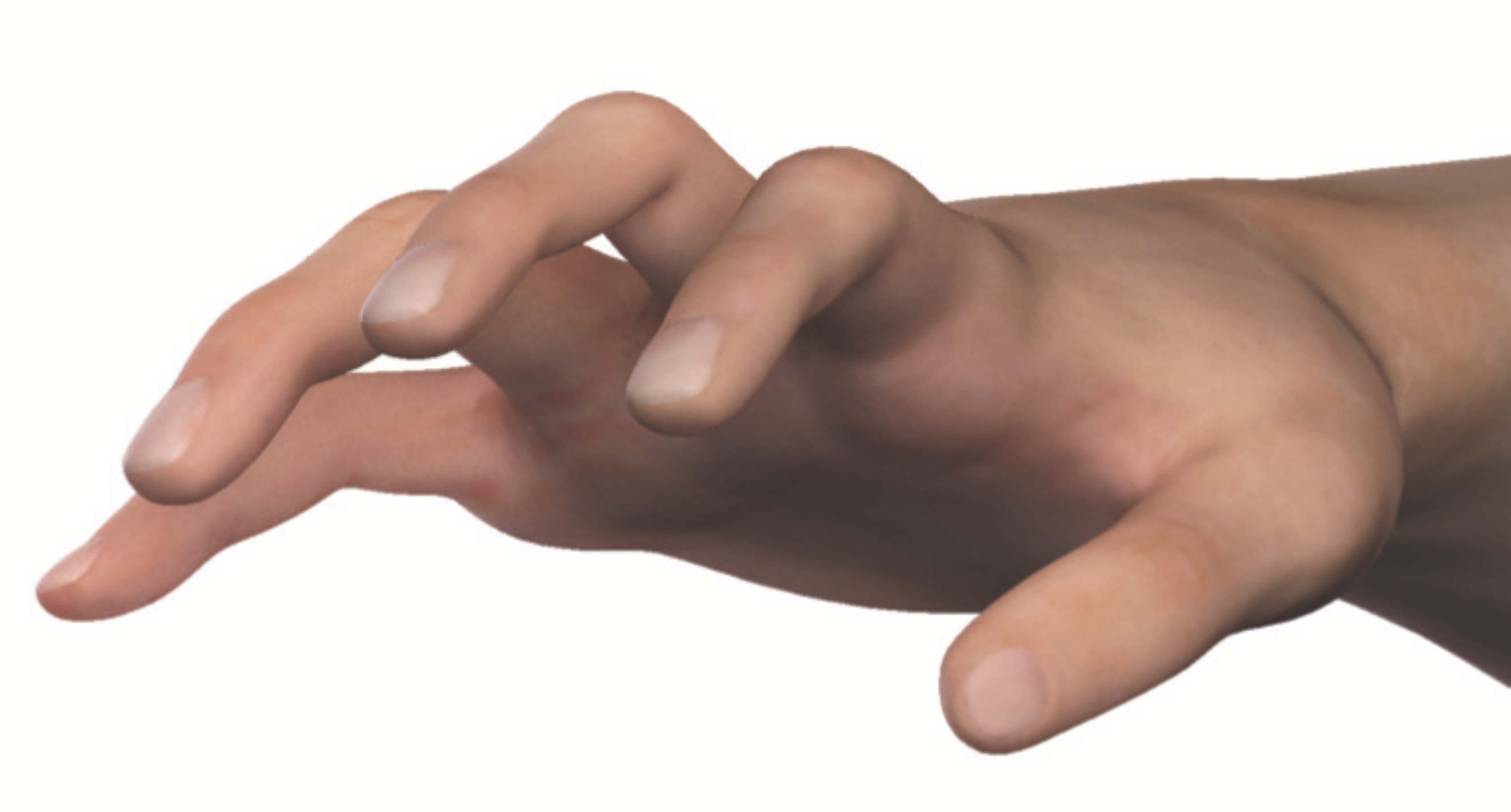} &
\includegraphics[width=0.185\textwidth]{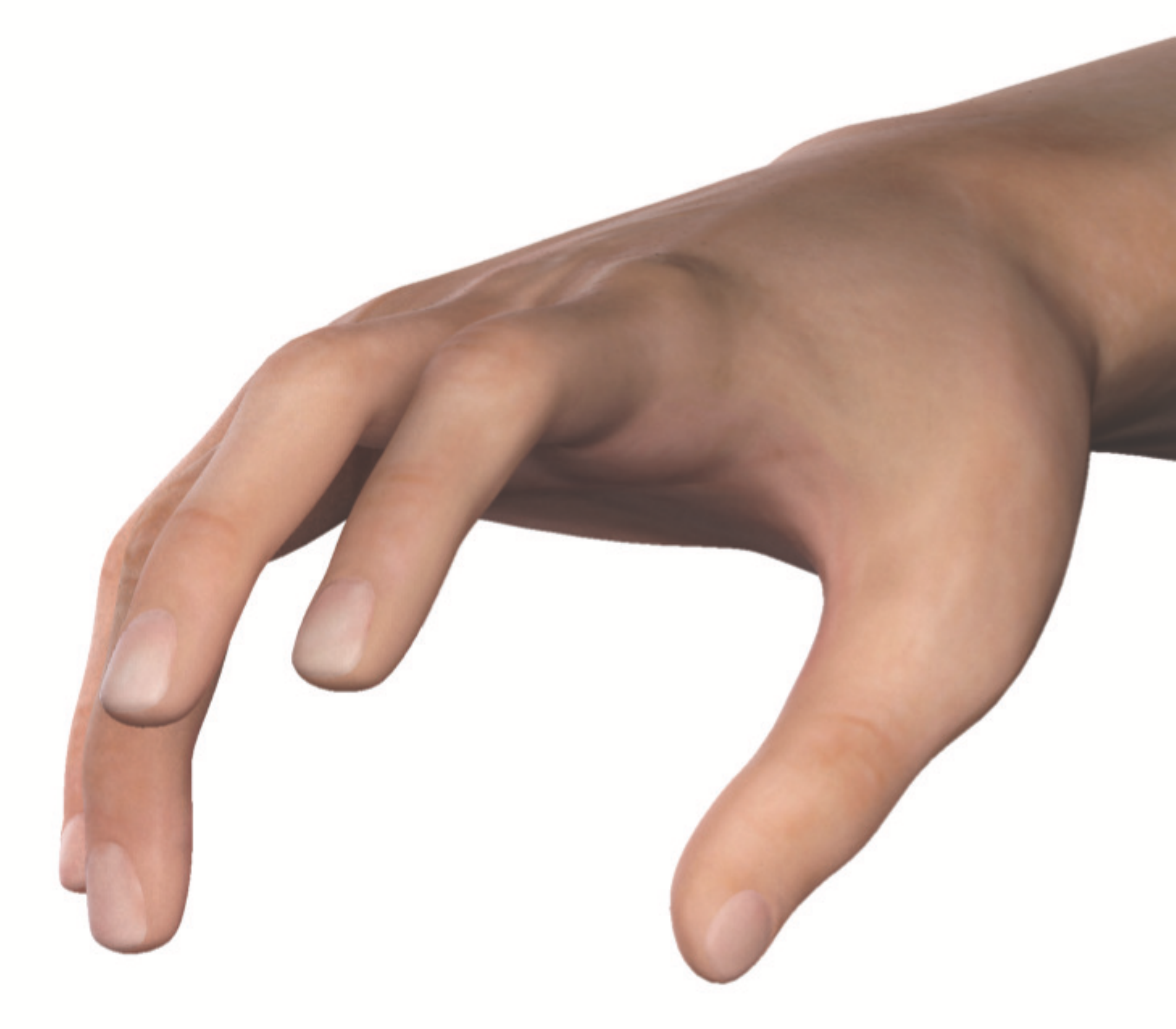} &
\includegraphics[width=0.165\textwidth]{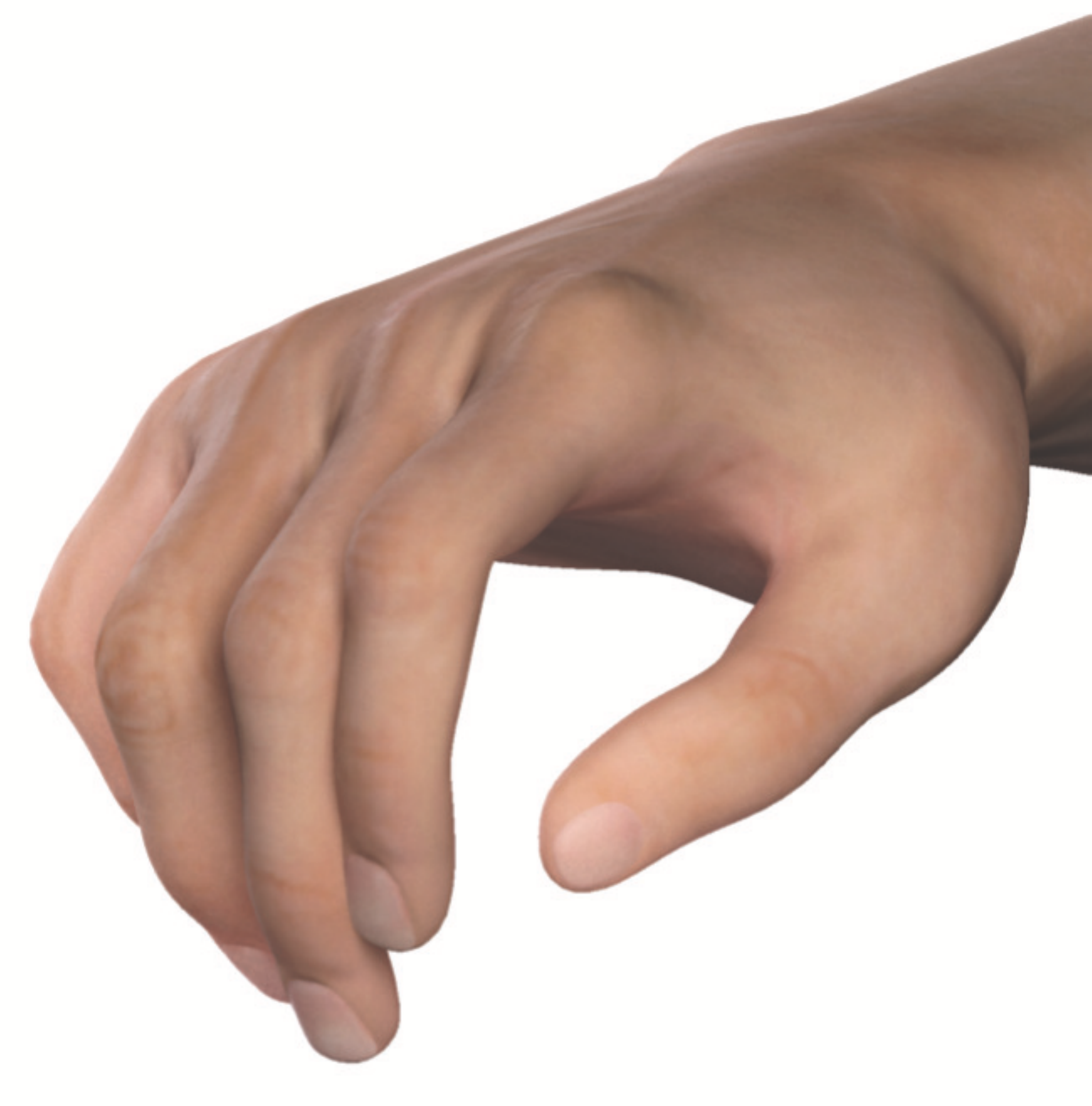} &
\includegraphics[width=0.2\textwidth]{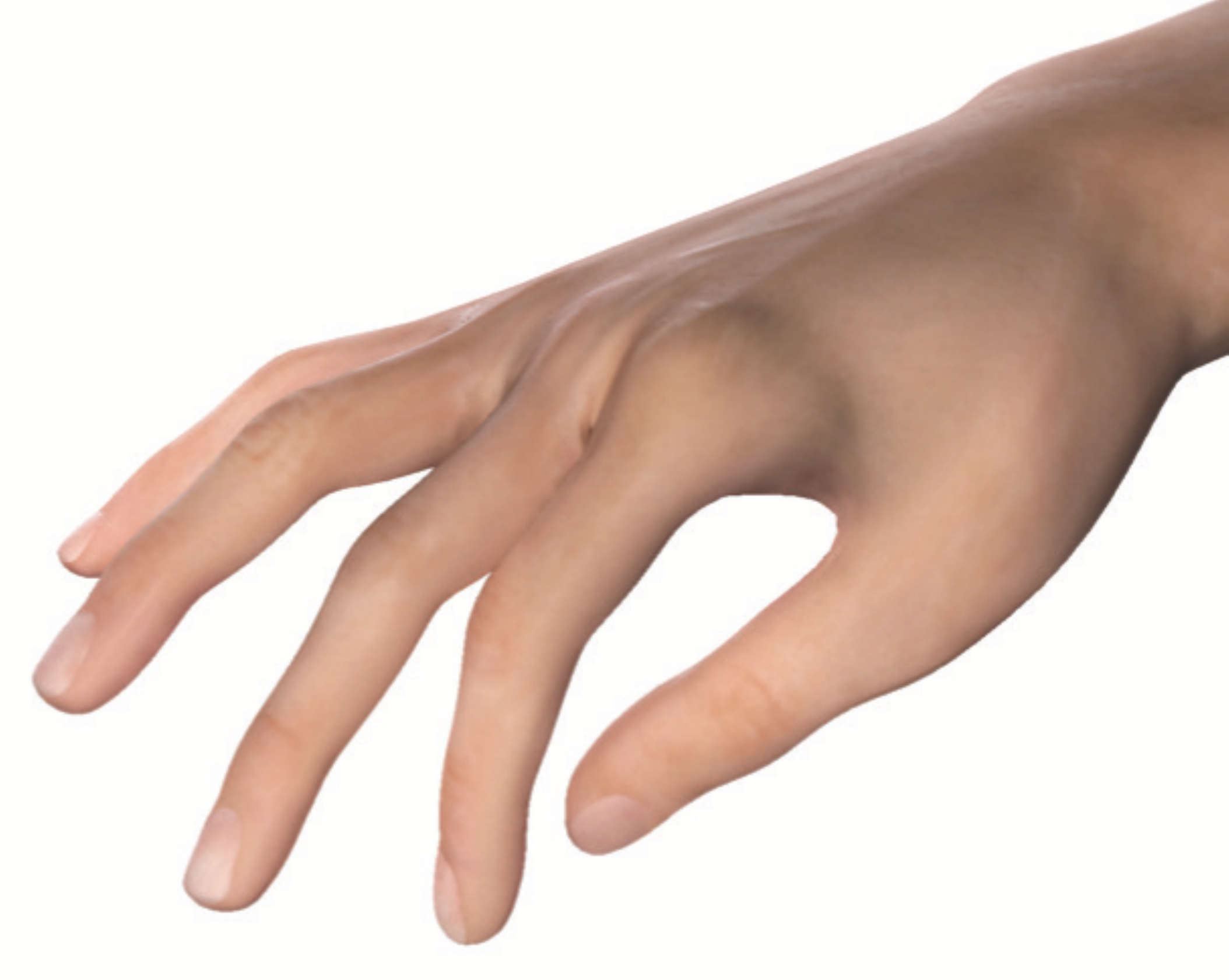}\\
\cline{2-5}
\cline{2-5}
\end{tabular}
\begin{tabular}[c]{c}
Posture estimation by using noise--free measures\\
\end{tabular}
\begin{tabular}[c]{|p{0.3cm}|p{2.4cm}|p{2.4cm}|p{2.4cm}|p{2.4cm}|}
\hline
\hline
{\rotatebox{90}{\mbox{MVE with $H_s$}}} &
\includegraphics[width=0.21\textwidth]{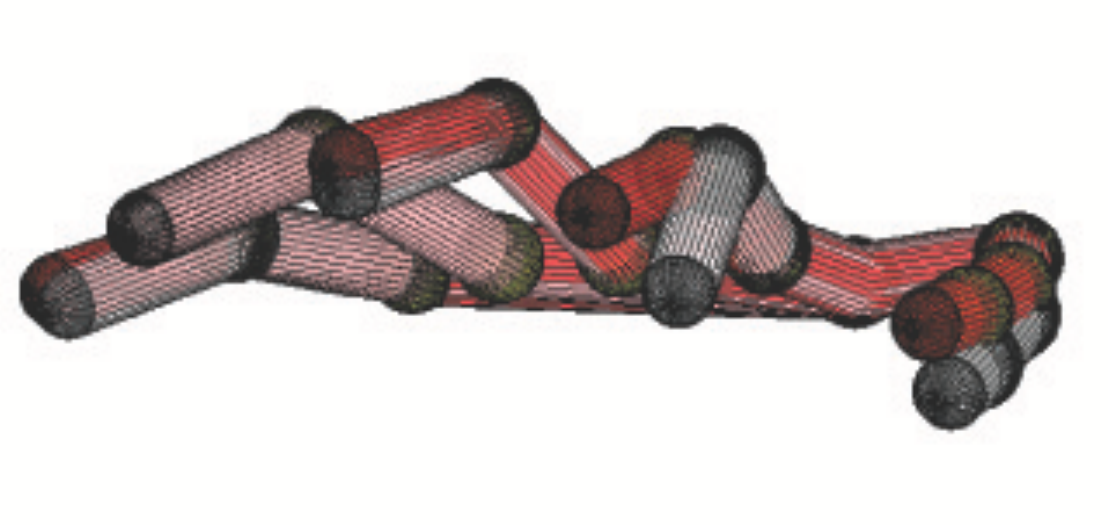} &
\includegraphics[width=0.2\textwidth]{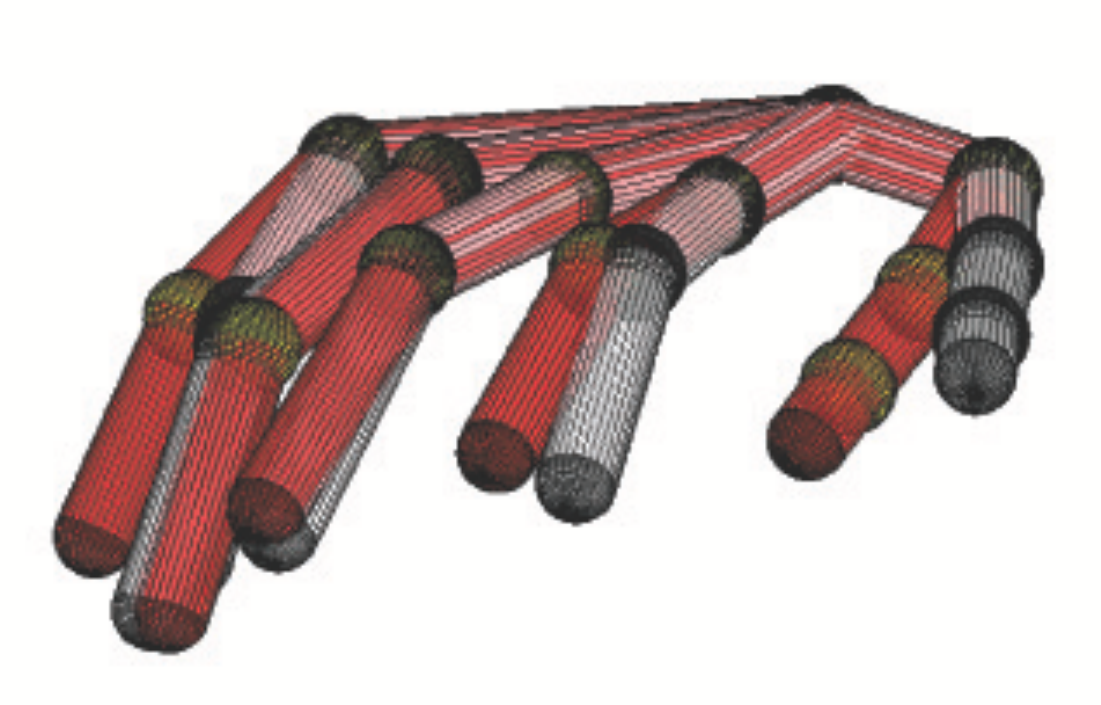} &
\includegraphics[width=0.175\textwidth]{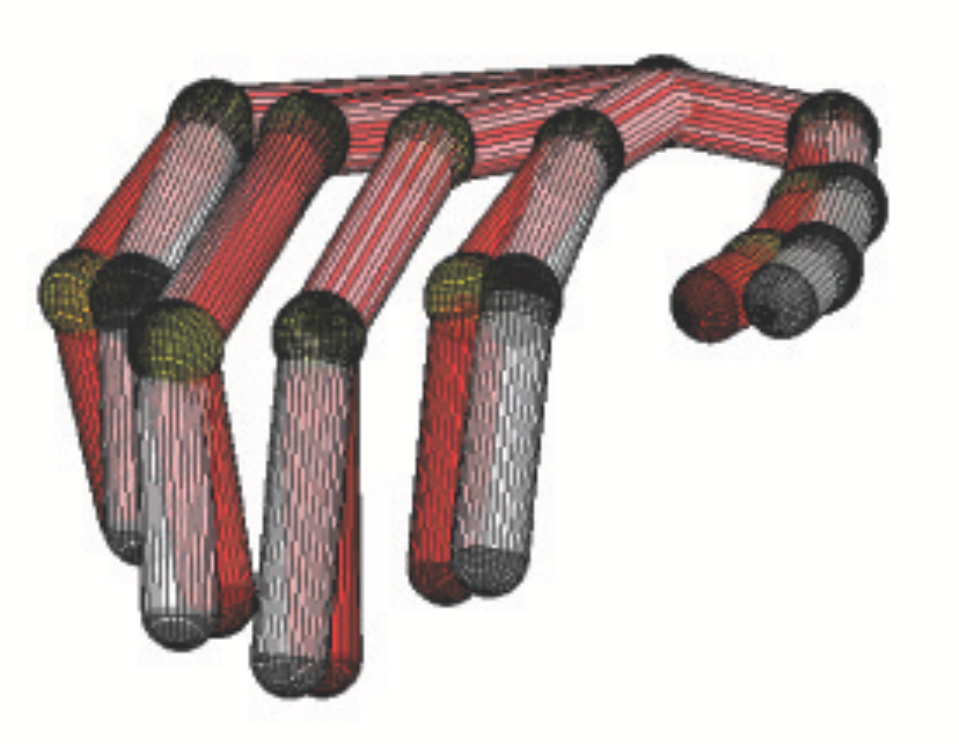} &
\includegraphics[width=0.19\textwidth]{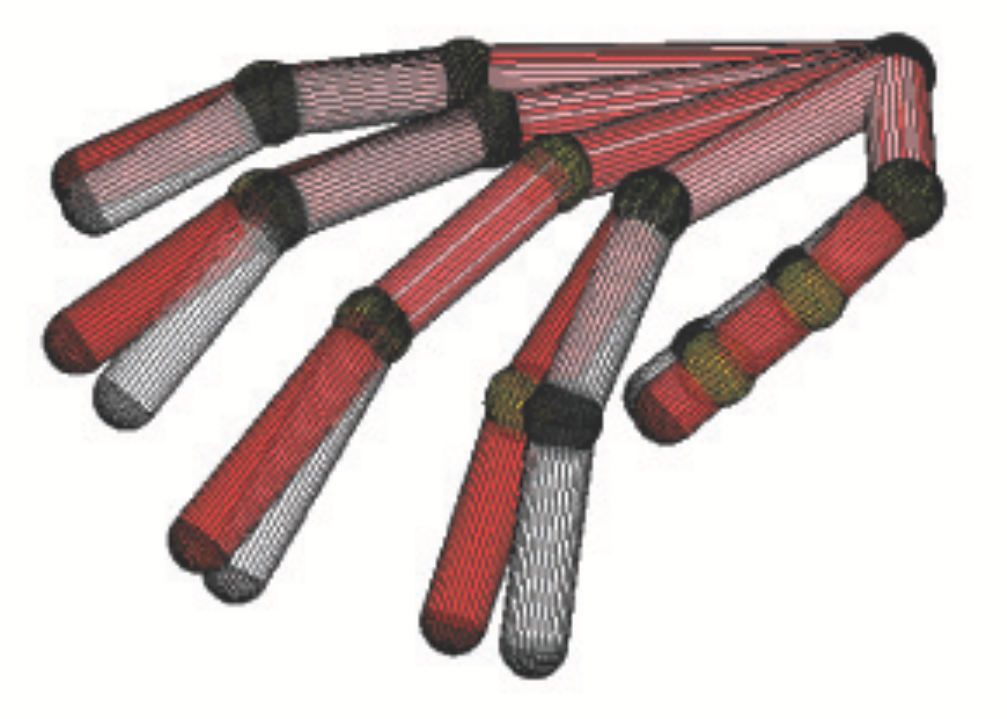}\\
\hline
{\rotatebox{90}{\mbox{MVE with $H^*_d$}}} &
\includegraphics[width=0.21\textwidth]{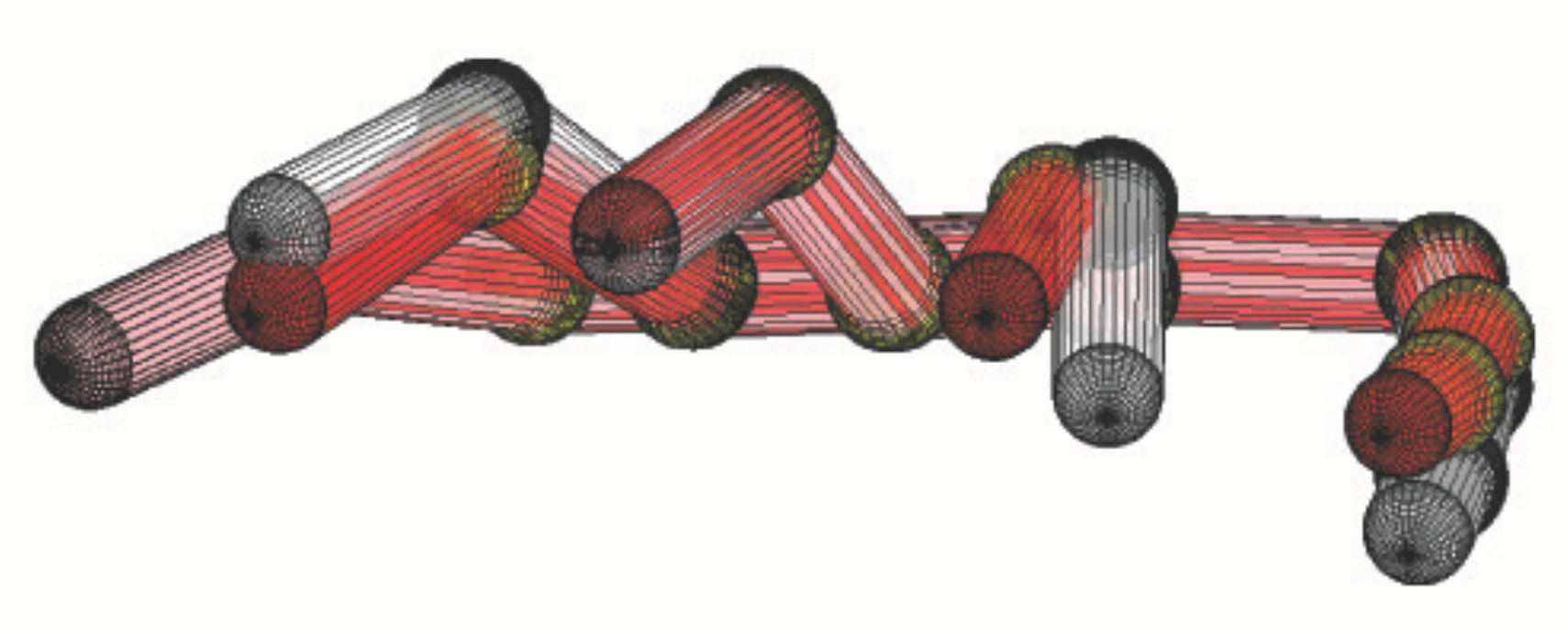} &
\includegraphics[width=0.2\textwidth]{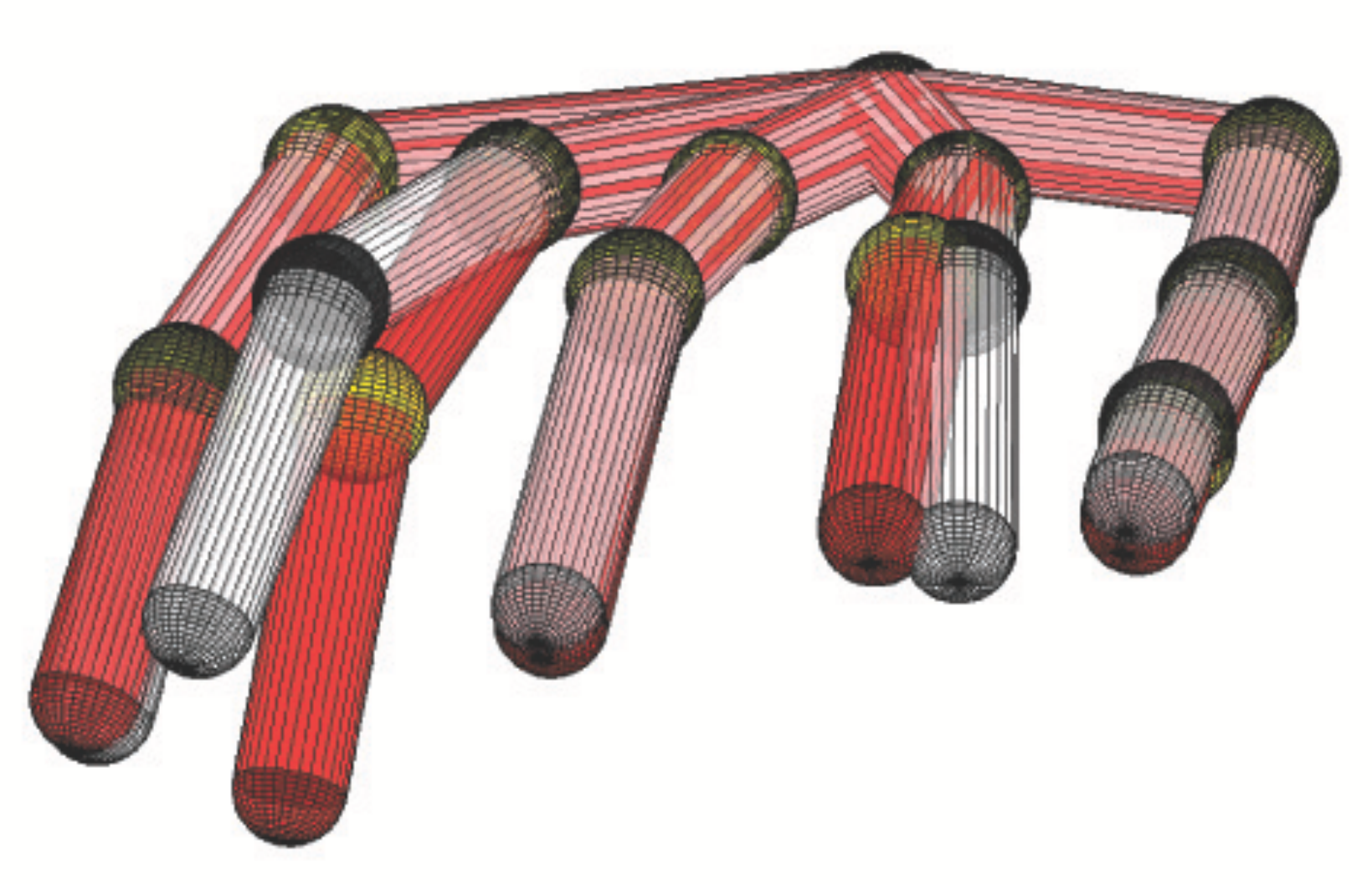} &
\includegraphics[width=0.175\textwidth]{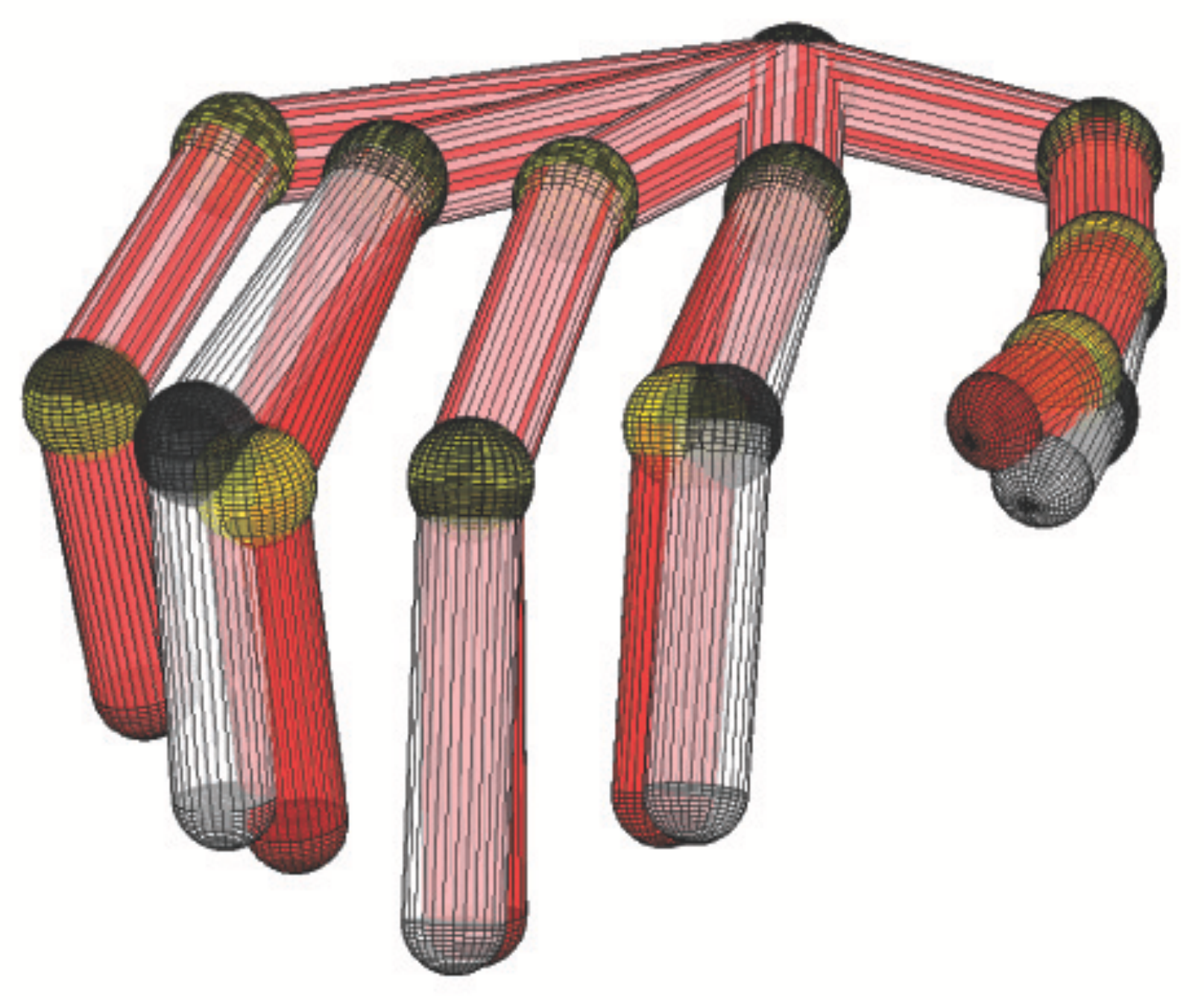} &
\includegraphics[width=0.21\textwidth]{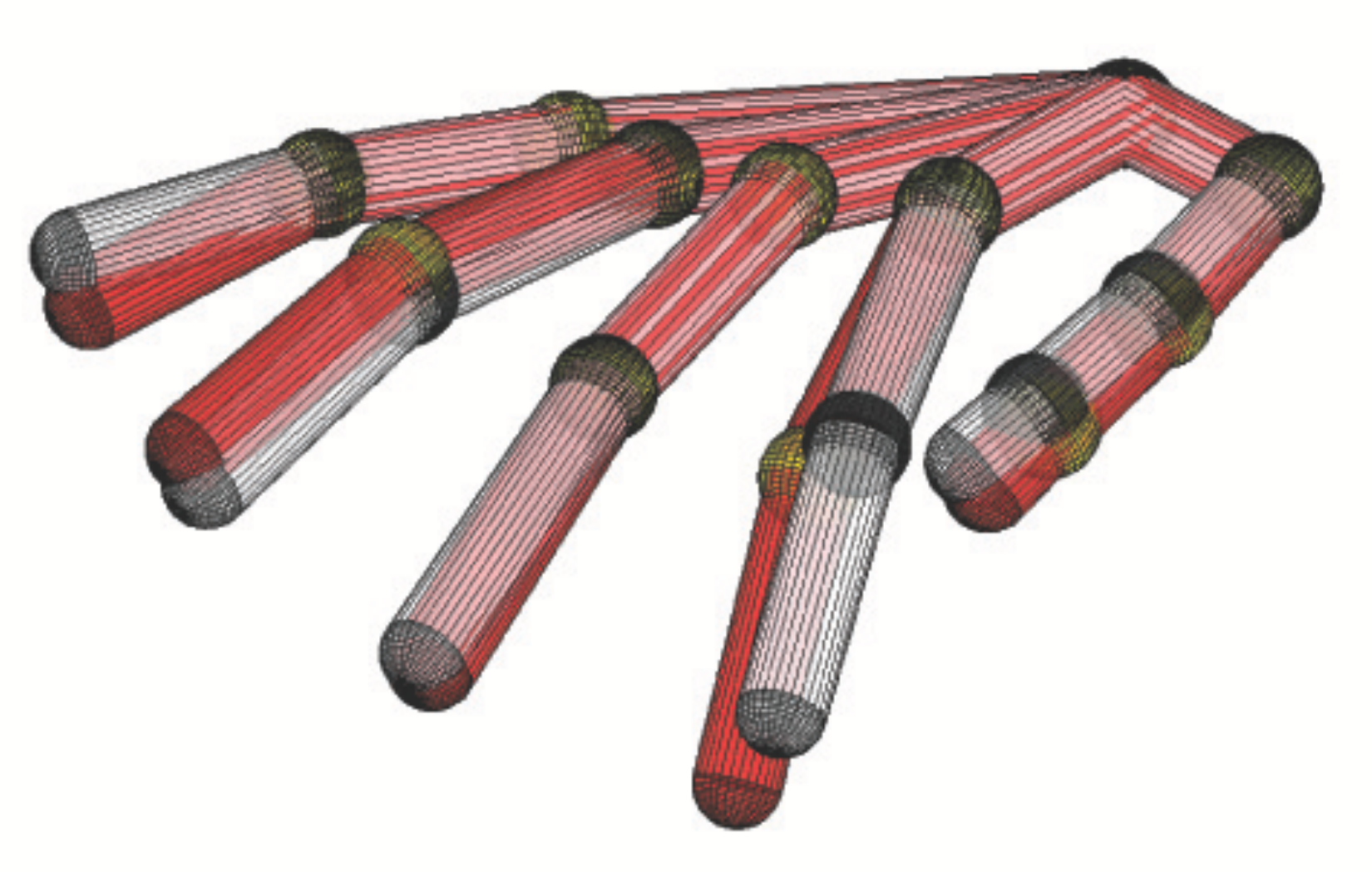}\\
\hline
\hline
\end{tabular}
\begin{tabular}[c]{c}
Posture estimation by using noisy measures\\
\end{tabular}
\begin{tabular}[c]{|p{0.3cm}|p{2.4cm}|p{2.4cm}|p{2.4cm}|p{2.4cm}|}
\hline
\hline
{\rotatebox{90}{\mbox{MVE with $H_s$}}} &
\includegraphics[width=0.22\textwidth]{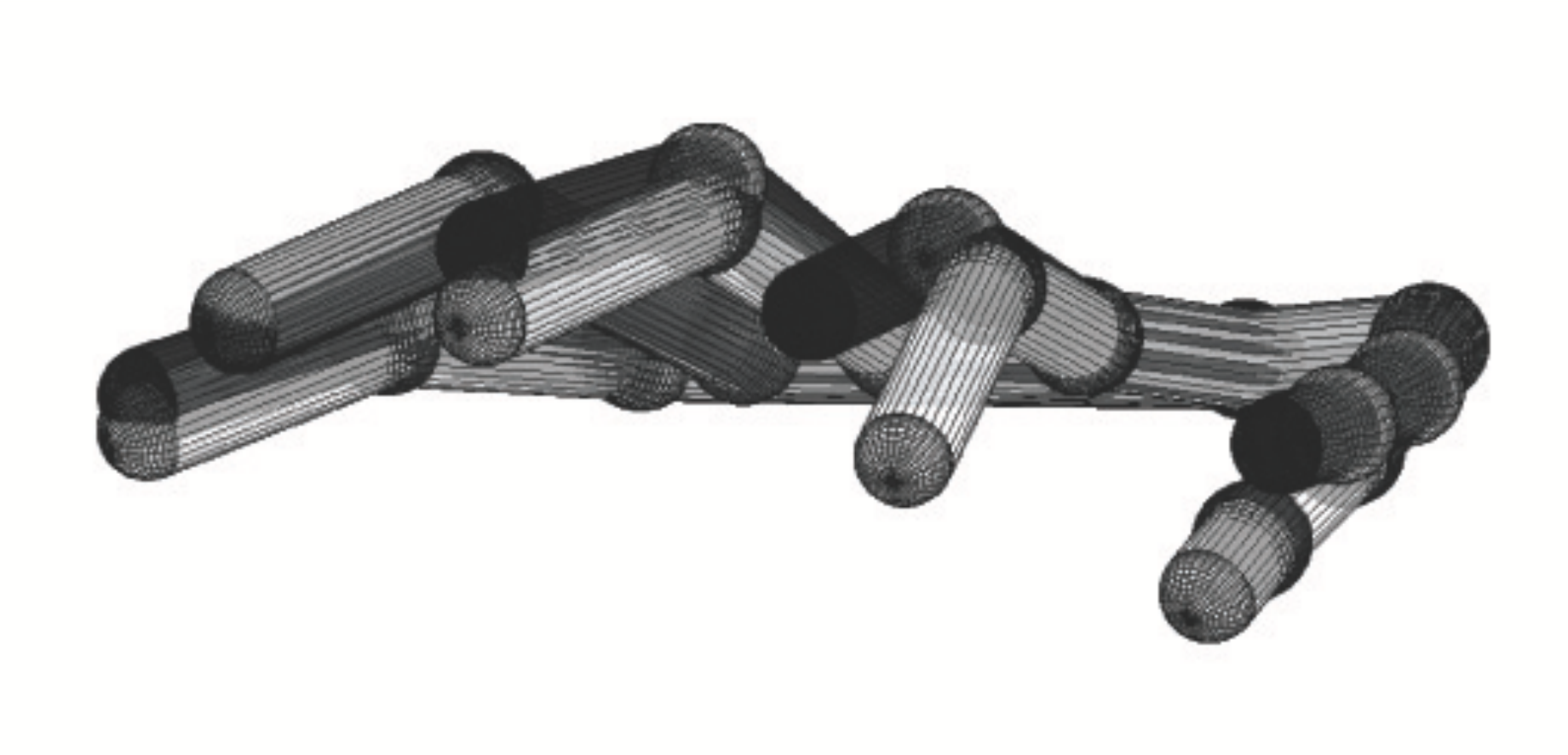} &
\includegraphics[width=0.21\textwidth]{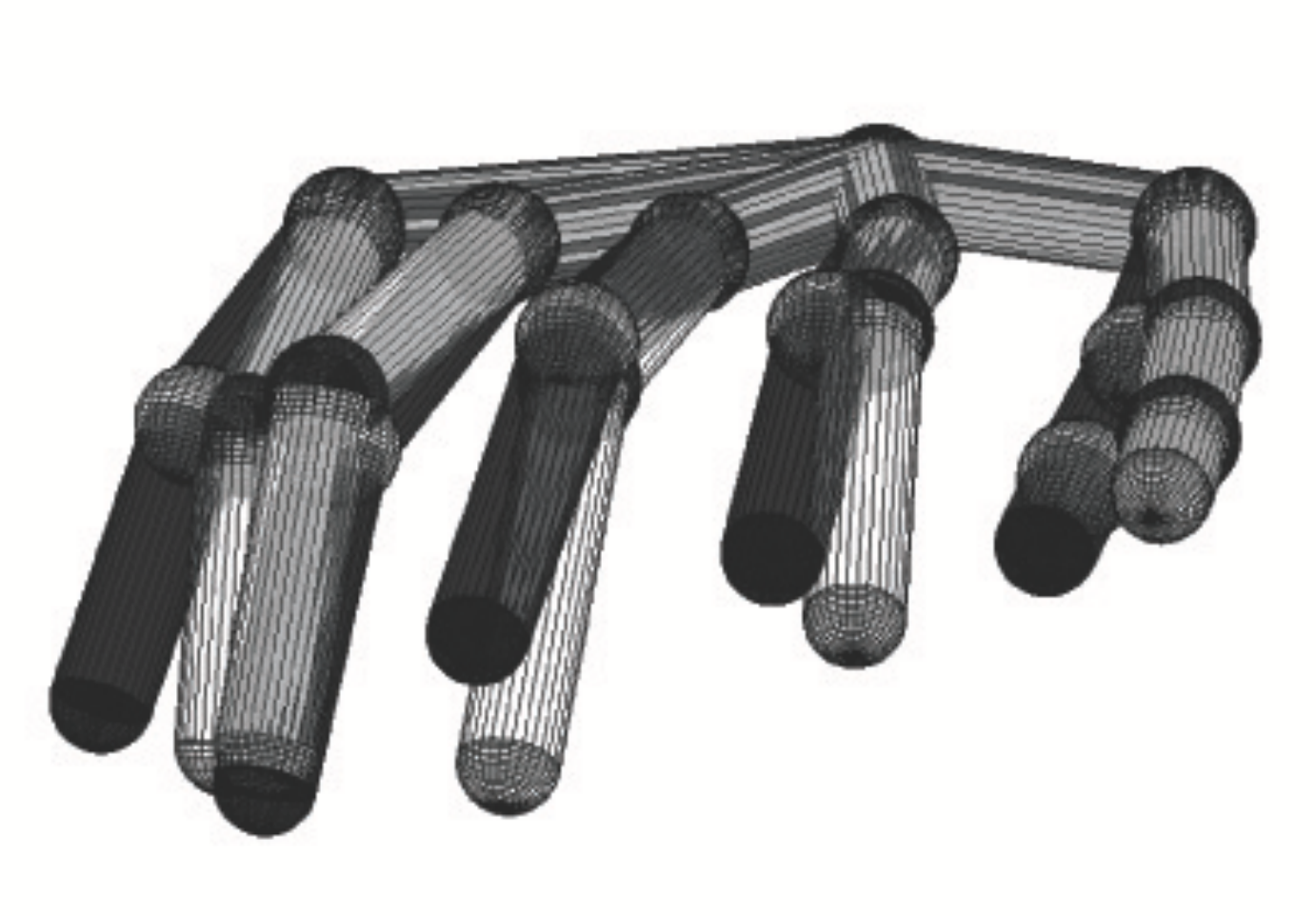} &
\includegraphics[width=0.18\textwidth]{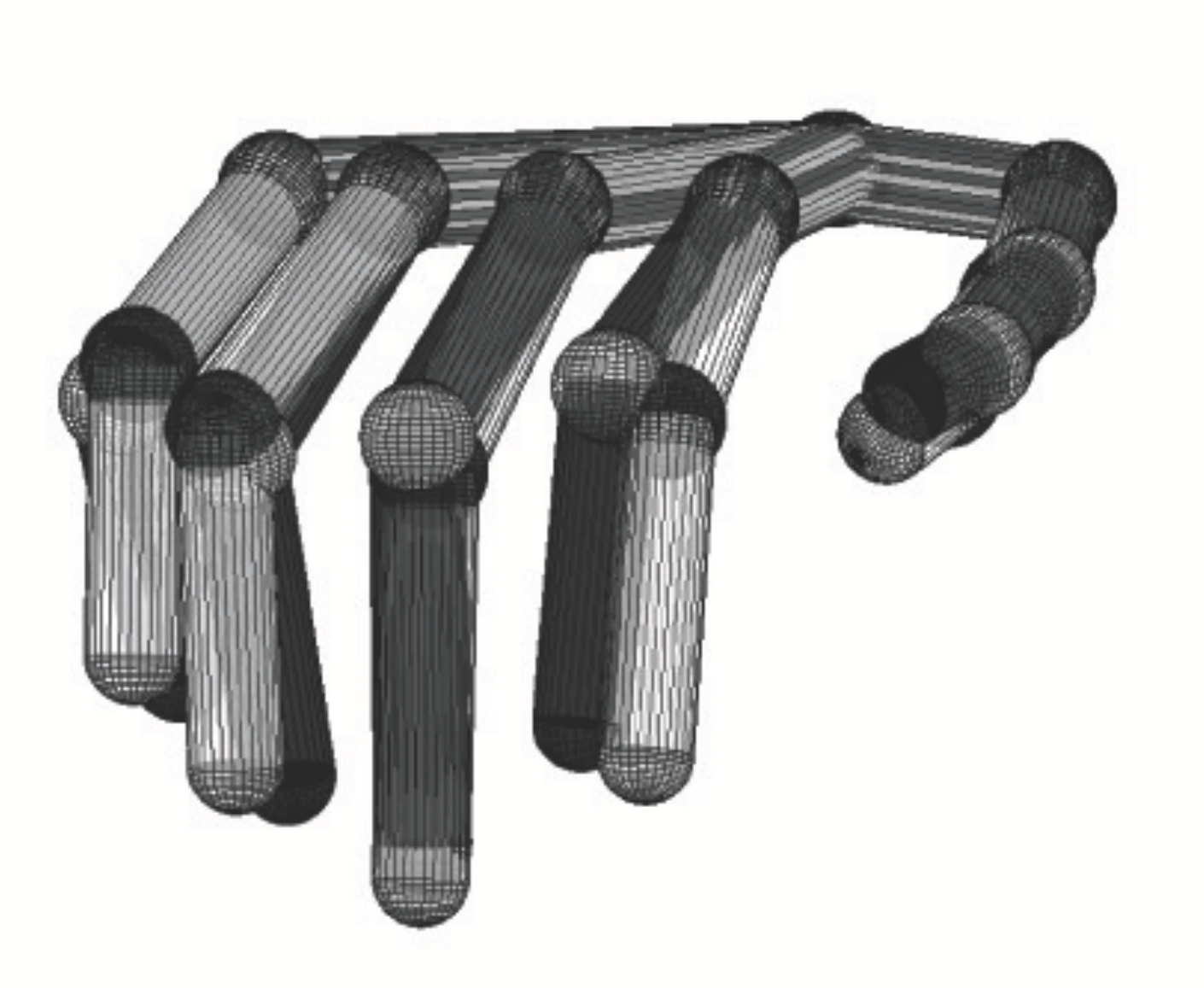} &
\includegraphics[width=0.2\textwidth]{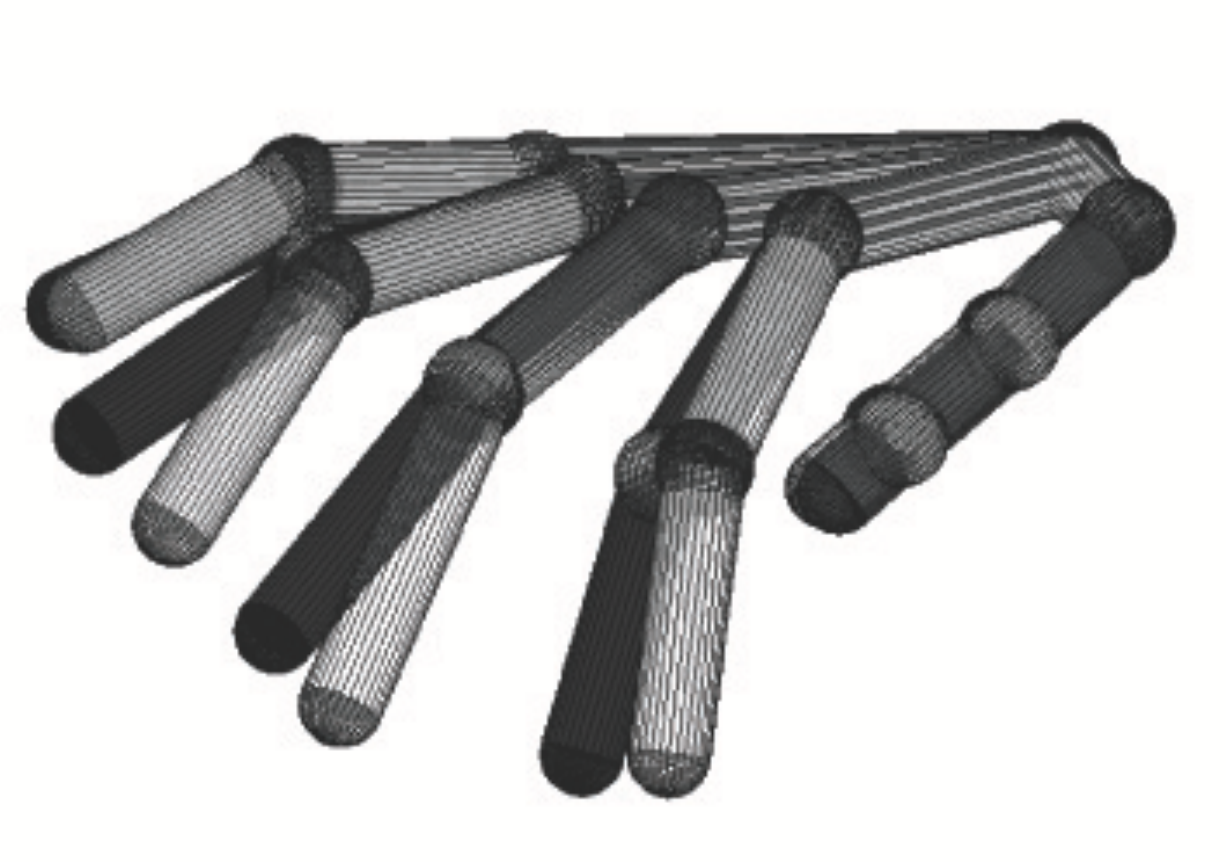}\\
\hline
{\rotatebox{90}{\mbox{MVE with $H^*_d$}}} &
\includegraphics[width=0.21\textwidth]{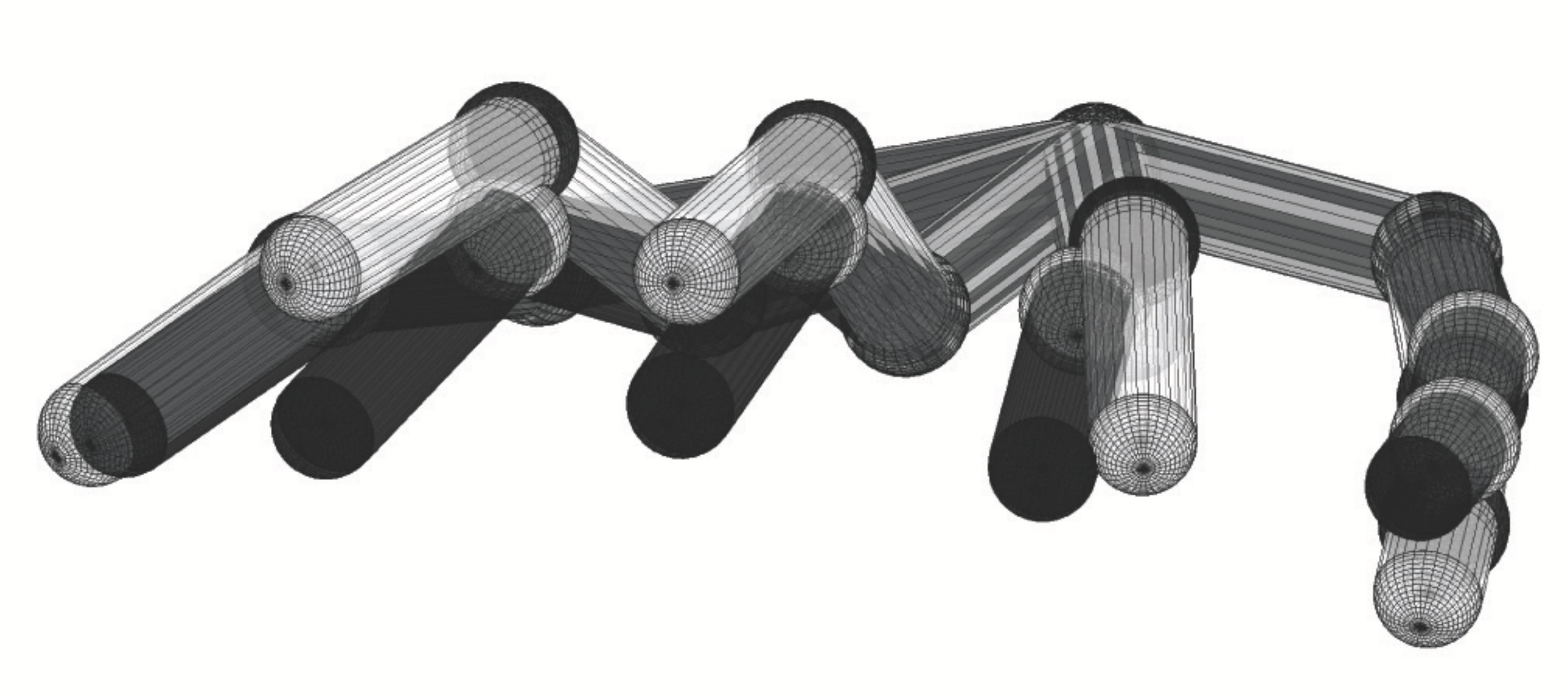} &
\includegraphics[width=0.21\textwidth]{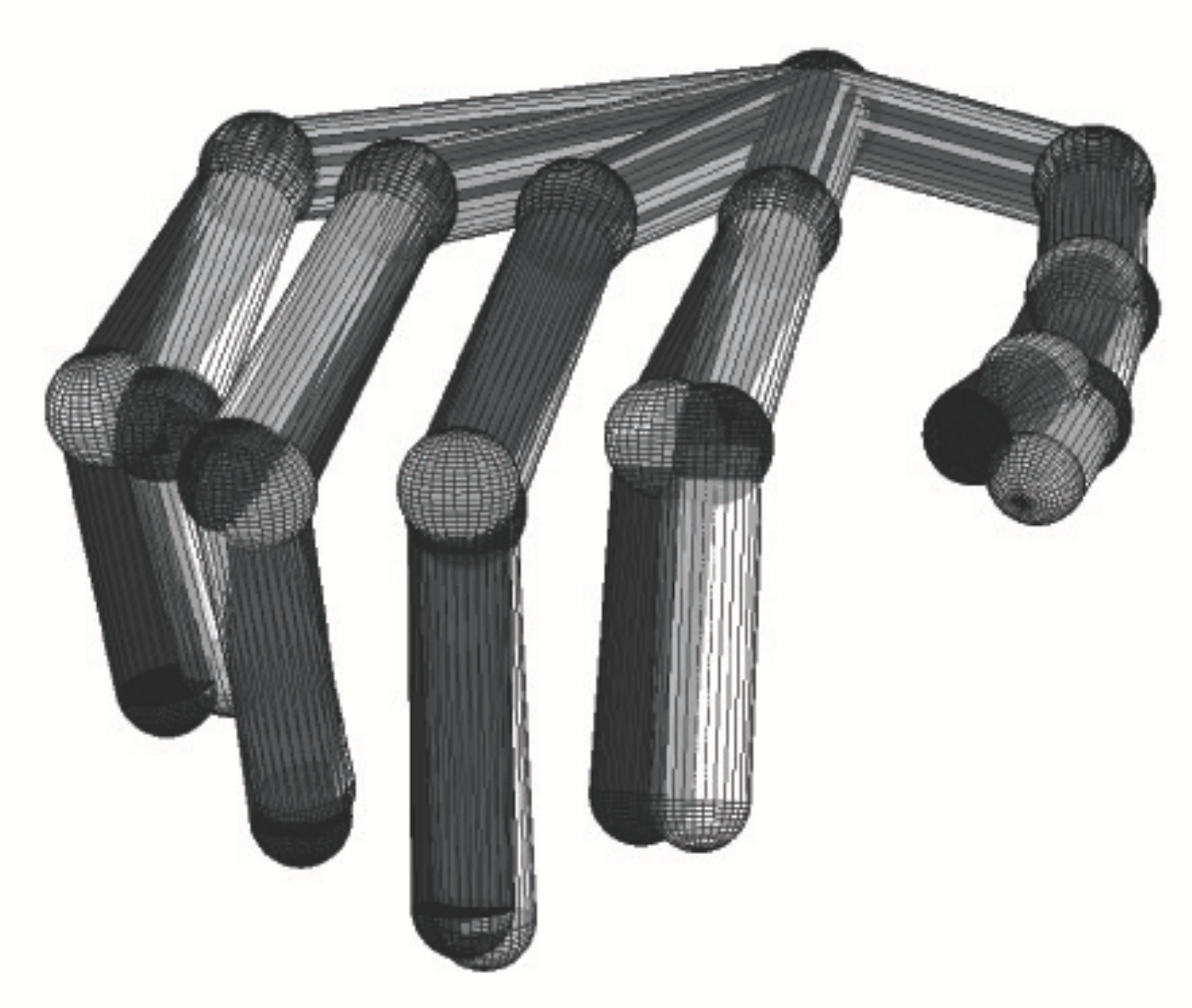} &
\includegraphics[width=0.18\textwidth]{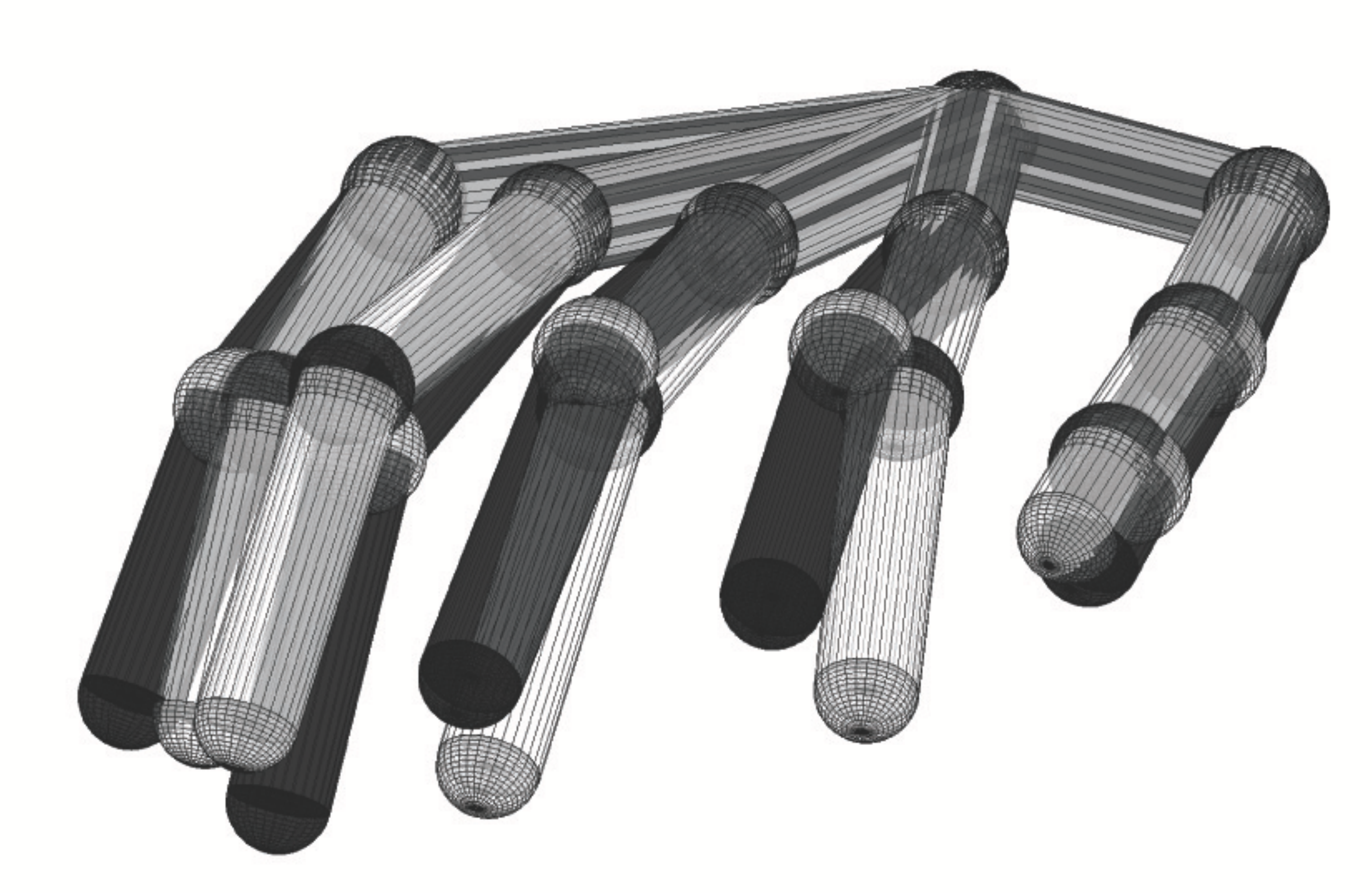} &
\includegraphics[width=0.2\textwidth]{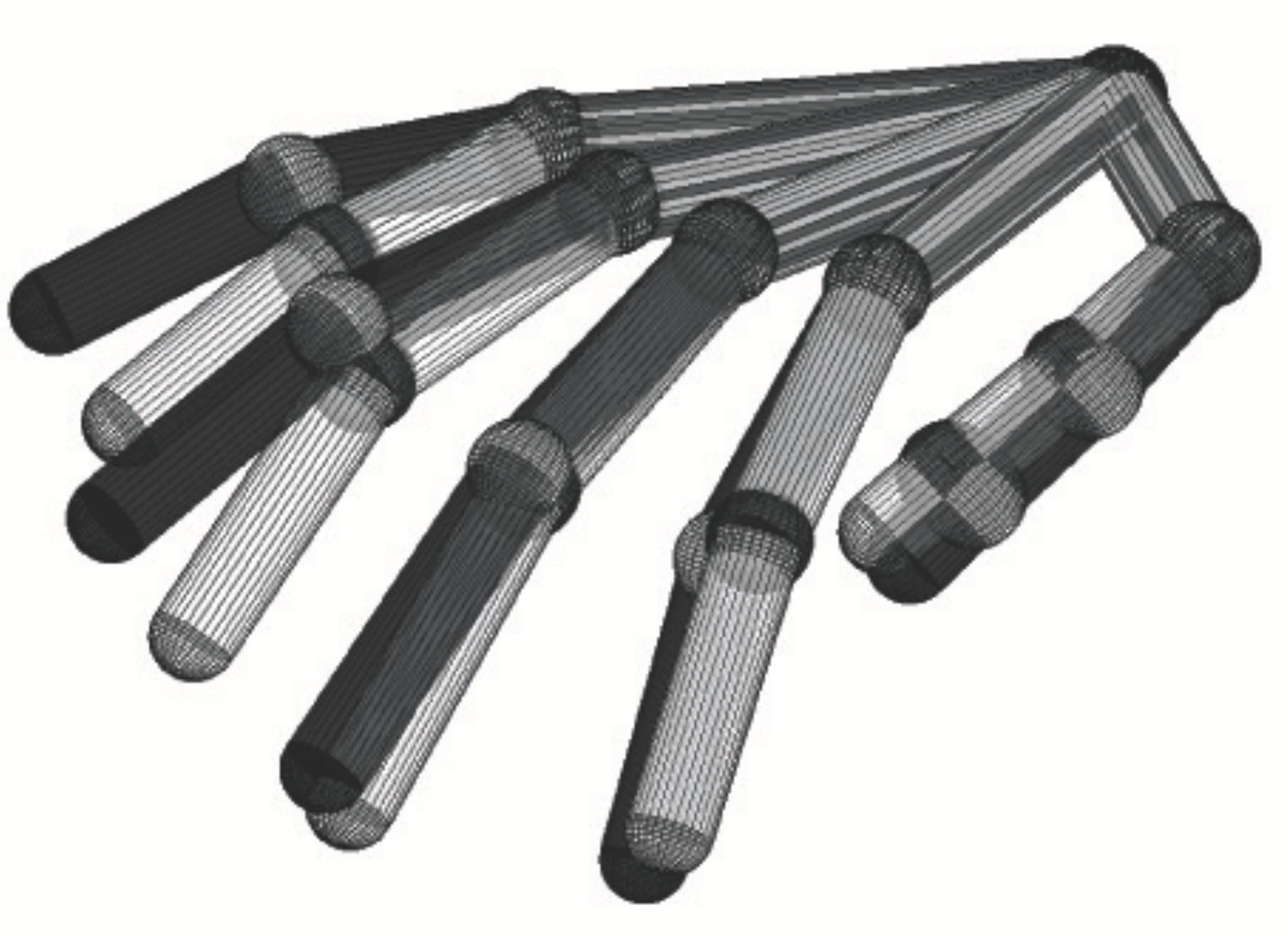}\\
\hline
\end{tabular}
\end{center}
\caption{Hand pose reconstructions MVE algorithm by using matrix $H_s$ which allows to measure $TM$, $IM$, $MM$, $RM$ and $LM$ and matrix $H^*_d$ which allows to measure $TA$, $MM$, $RP$, $LA$ and $LM$  (cf.~figure~\ref{fig:KinMod}). In color the real hand posture whereas in white the estimated one.}
\label{fig:PoseEst_WithAndWithoutNoise}
\end{figure*}

\subsubsection{Noise-Free Measures}

In terms of average absolute estimation pose errors ([$^{\circ}$]), performance obtained with $H^*_d$ is always better than the one exhibited by $H_s$ (3.67$\pm$0.93 vs.~6.69$\pm$2.38). Moreover, $H^*_d$ exhibits smaller maximum error than the one achieved with $H_s$ (i.e.~$8.25^{\circ}$ for $H^*_d$ vs.~$13.18^{\circ}$ for $H_s$). Statistical differences between results from $H_s$ and $H^*_d$ are found ($\text{p}~\simeq~0$, $T_{neq}$). In table~\ref{tab:totPVALUESFree35} average absolute estimation errors with their corresponding standard deviations for each DoF are reported. For the estimated DoFs, performance with $H^*_d$ is always better or not statistically different from the one referred to $H_s$. Maximum estimation errors underline cases where $H_s$ furnishes smaller values and vice versa, since they strictly depend on peculiar poses; however, results from the two matrices are globally comparable.
\begin{table}[t!]
\centering
\includegraphics[width=0.7\columnwidth]{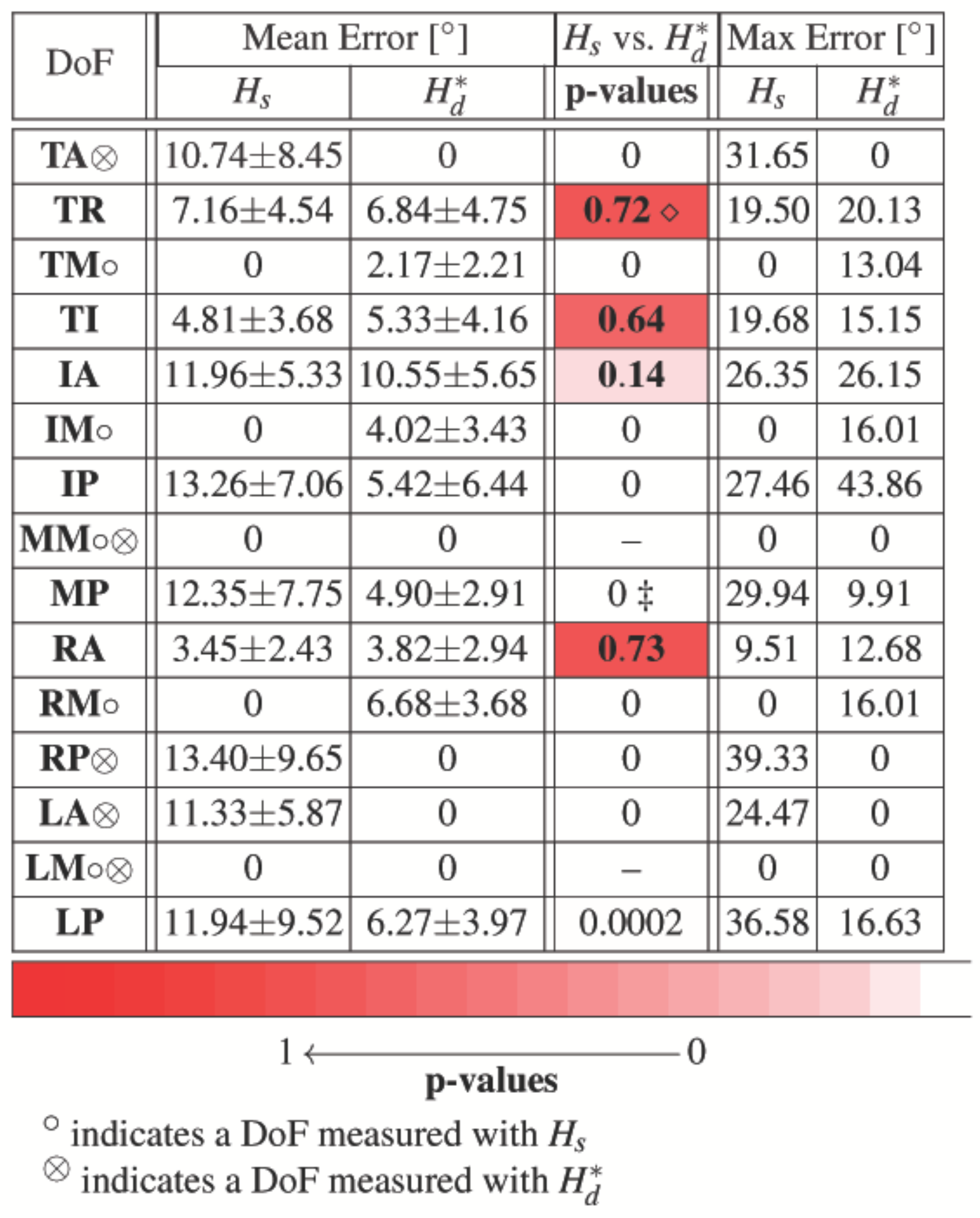}
\caption{Average estimation errors and standard deviation for each DoF $[\circ]$ for the simulated acquisition considering $H_s$ and $H^*_d$ both with five noise free measures. Maximum errors are also reported as well as p-values from the evaluation of DoF estimation errors between $H_s$ and $H^*_d$.
$\diamond$ indicates $T_{eq}$ test. $\ddagger$ indicates $T_{neq}$ test. When no symbol appears near the tabulated values, $U$ test is used.
$\bold{Bold}$ value indicates no statistical difference between the two methods under analysis at 5\% significance level.
When the difference is significative, values are reported with a $10^{-4}$ precision. p-values less than $10^{-4}$ are considered equal to zero. Symbol ``--'' is used for those DoFs which are measured by both $H_s$ and $H^*_d$.}
\label{tab:totPVALUESFree35}
\end{table}

\begin{figure}[t!]
\centering
\includegraphics[width=0.95\columnwidth]{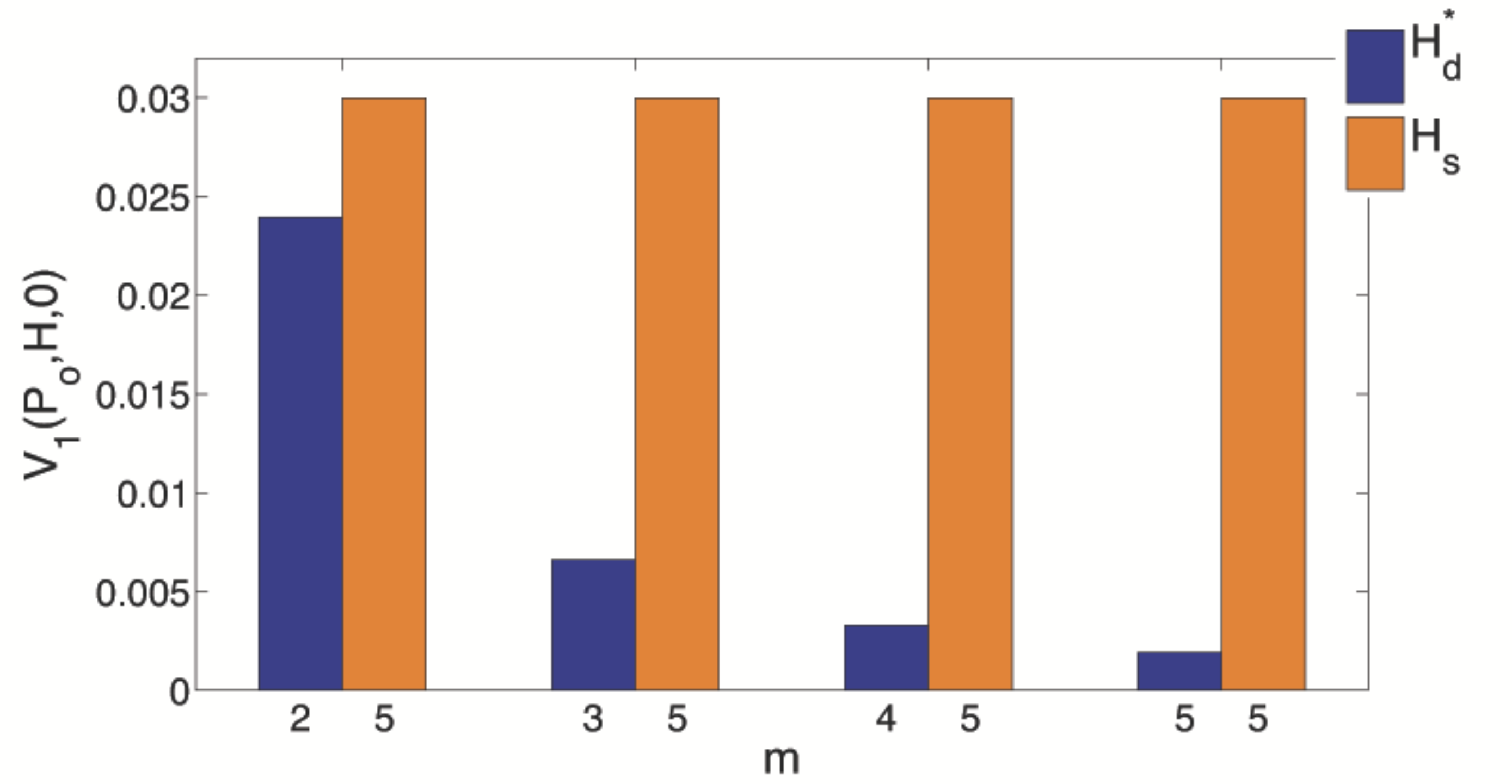}
\caption{Squared Frobenius norm for the {\em a posteriori} covariance matrix of $H_s$ with $m=5$ measures, and $H^*_d$ with $m=2,\,3,\,4,\,5$ measures, in case of noise--free measures.}
\label{fig:parOpt}
\end{figure}

\begin{figure}[t!]
\centering
\includegraphics[width=0.95\columnwidth]{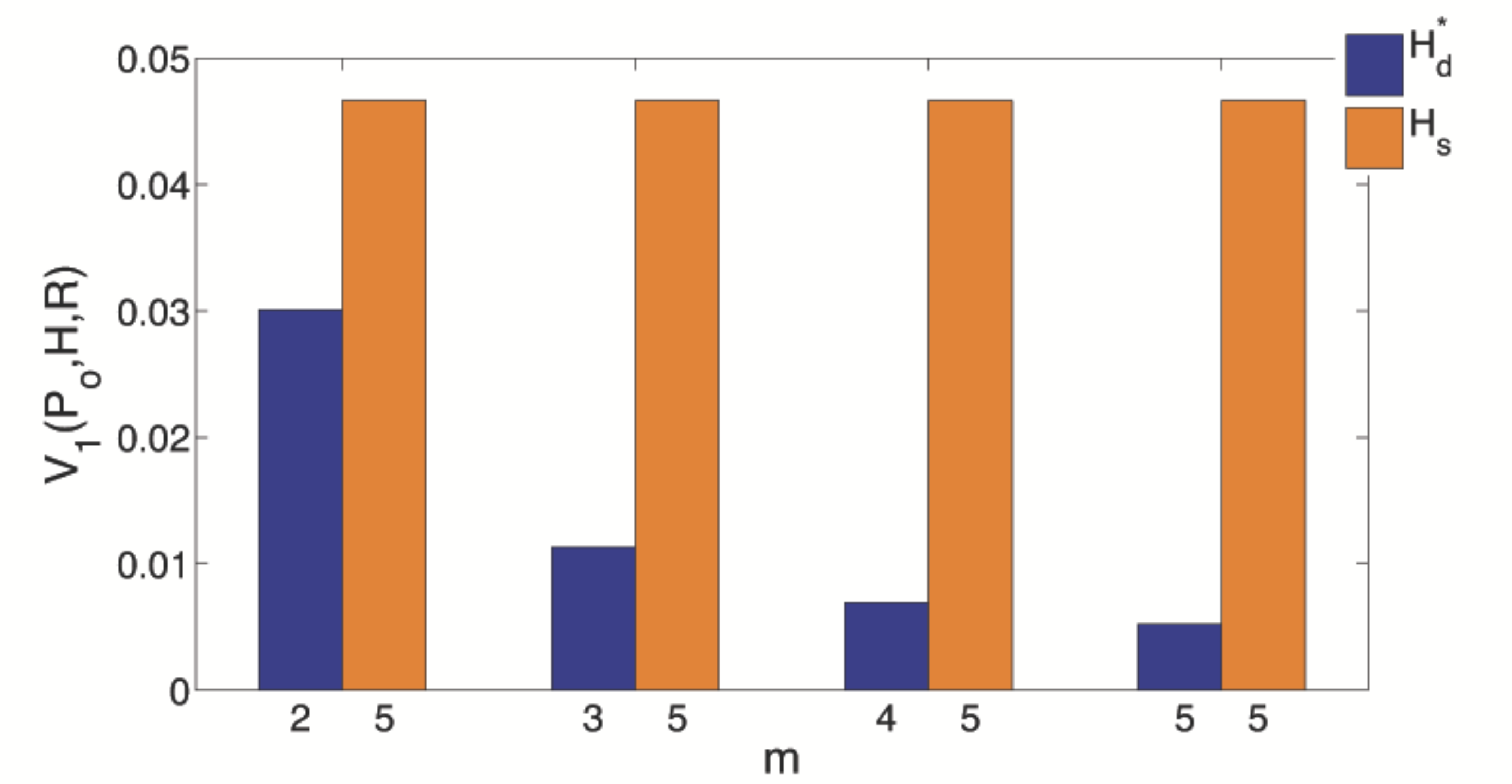}
\caption{Squared Frobenius norm for the {\em a posteriori} covariance matrix of $H_s$ with $m=5$ measures, and $H^*_d$ with $m=2,\,3,\,4,\,5$ measures, in case of noisy measures.}
\label{fig:parOptR}
\end{figure}

In figure~\ref{fig:parOpt}, squared Frobenius norm for the {\em a posteriori} covariance matrix of $H_s$ with $m=5$ measures, and $H^*_d$ with $m=2,3,4,5$ measures, in case of noise-free measures is reported. Notice that squared Frobenius norm is significantly smaller in the optimal case, even when a reduced number of measures is considered.

\subsubsection{Noisy Measures}

In case of noise, performance in terms of average absolute estimation pose errors ([$^{\circ}$]) obtained with $H^*_d$ is better than the one exhibited by $H_s$ (5.96$\pm$1.42  vs.~8.18$\pm$2.70). Moreover, maximum pose error with $H^*_d$ is the smallest ($9.30^{\circ}$  vs.~$15.35^{\circ}$ observed with $H_s$). Statistical difference between results from $H_s$ and $H^*_d$ are found ($\text{p}$=0.001, $T_{neq}$).

In table~\ref{tab:totPVALUESNoise35} average absolute estimation error with standard deviations are reported for each DoF. For the estimated DoFs, performance with $H^*_d$ is always better or not statistically different from the one referred to $H_s$. Maximum estimation errors with $H^*_d$ are usually inferior to the ones obtained with $H_s$.

\begin{table}[t!]
\centering
\includegraphics[width=0.7\columnwidth]{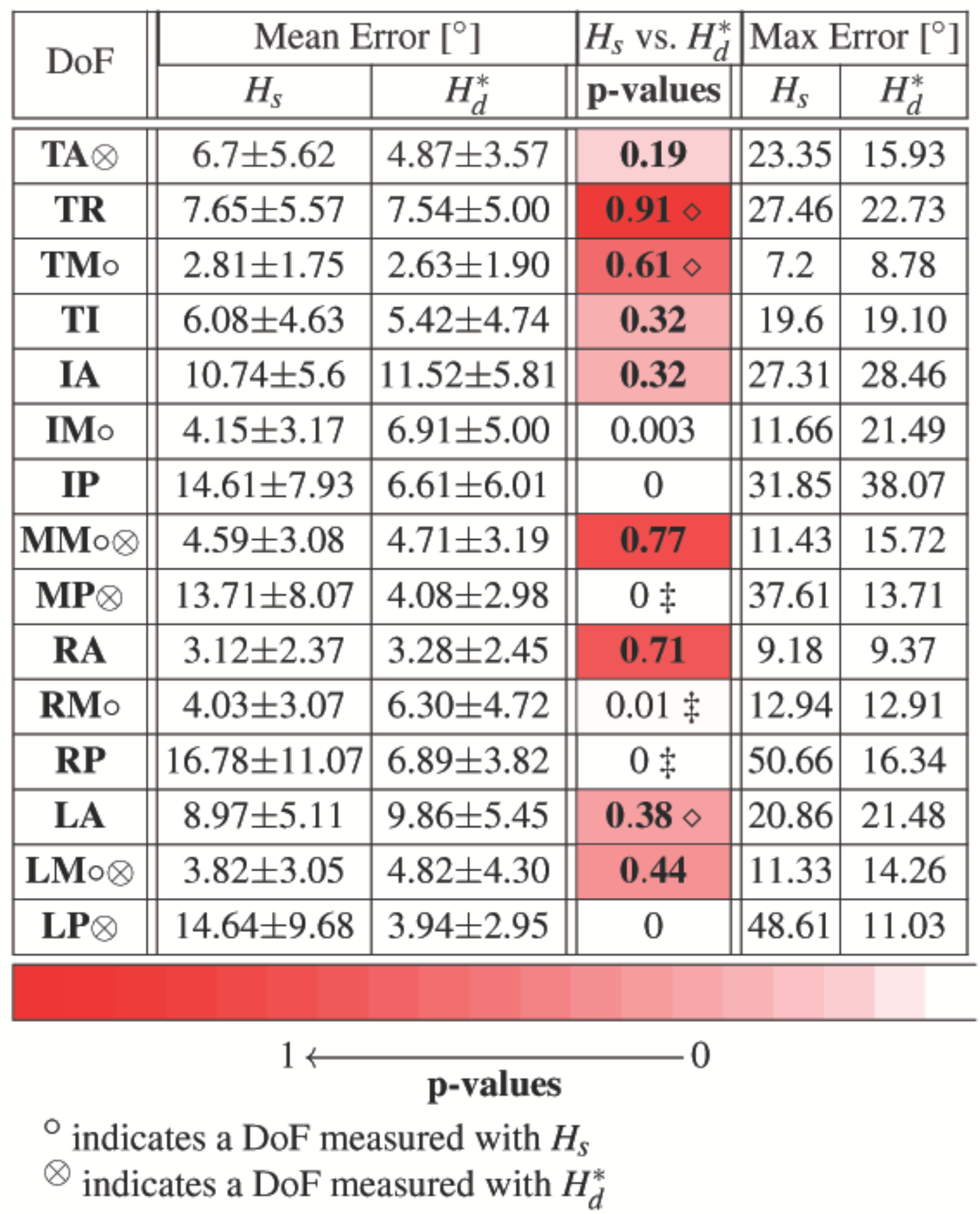}
\caption{Average estimation errors and standard deviation for each DoF $[\circ]$ for the simulated acquisition considering $H_s$ and $H^*_d$ both with five noisy measures. Maximum errors are also reported as well as p-values from the evaluation of DoF estimation errors between $H_s$ and $H^*_d$.
$\diamond$ indicates $T_{eq}$ test. $\ddagger$ indicates $T_{neq}$ test. When no symbol appears near the tabulated values, $U$ test is used.
$\bold{Bold}$ value indicates no statistical difference between the two methods under analysis at 5\% significance level.
When the difference is significative, values are reported with a $10^{-4}$ precision. p-values less than $10^{-4}$ are considered equal to zero. Symbol ``--'' is used for those DoFs which are measured by both $H_s$ and $H^*_d$.}
\label{tab:totPVALUESNoise35}
\end{table}
Figure~\ref{fig:parOptR} shows the squared Frobenius norm for the {\em a posteriori} covariance matrix of $H_s$ with $m=5$ measures, and $H^*_d$ with $m=2,\,3,\,4,\,5$ measures, in case of noise. Also in this situation, squared Frobenius norm is significantly smaller in the optimal case, even if a reduced number of measures is considered, thus suggesting that an optimal design leading to error statistics minimization can be achieved using optimal matrix with an inferior number of measured DoFs w.r.t.~$H_s$. Notice that in this case, squared Frobenius norm values are larger than the corresponding ones obtained in absence of noise, as expected.

Finally, in figure~\ref{fig:PoseEst_WithAndWithoutNoise} some reconstructed poses with MVE algorithm are reported by using both $H_s$ and $H^*_d$ measurement matrix, with and without additional noise. Under a qualitative point of view, what is noticeable is that reconstructed poses are not far from the real ones for
both measurement matrices. Moreover, it is not surprising that some poses seem to be estimated in a better manner using $H_s$ and vice versa, even if from the previously described statistical results $H^*_d$ provides best average performance. Indeed, MVE methods are thought to minimize error statistics rather than worst-case sensing errors related to peculiar poses~\cite{Bicchican}.

\section{Conclusions}

In this paper, optimal design of sensing glove has been proposed on the basis of the minimization of the \emph{a posteriori} covariance matrix as it results from the estimation procedure described in~\cite{Bianchi_etalI}. Optimal solution are described for the continuous, discrete and hybrid case.

In the continuous sensing case, optimal measures are individuated by principal components of the \emph{a priori} covariance matrix, thus suggesting the importance of postural synergies not only for hand control.

The reconstruction performance obtained by combining the estimation technique proposed in~\cite{Bianchi_etalI} and the optimal design proposed in this paper is significantly improved if compared with non-optimal measure case. Therefore,~\cite{Bianchi_etalI} and~\cite{Bianchi_etalII} provide a complete procedure to enhance the performance and for a more effective development of both sensorization systems for robotic hands and active touch sensing systems, which can be used in a
wide range of applications, ranging from virtual reality to
tele-robotics and rehabilitation. Moreover, by optimizing the number and location of sensors the production costs can be further reduced without loss of performance, thus increasing device diffusion.

\section*{Acknowledgment}
Authors gratefully acknowledge Marco Santello and Lucia Pallottino for the inspiring discussion and useful suggestions.

\bibliographystyle{apalike}

\begin{thebibliography}{1}

	\bibitem[Bianchi et~al., 2012a]{Bianchi_etalII}
	Bianchi, M., Salaris, P., and Bicchi, A. (2012a).
	\newblock Synergy-based hand pose sensing: Optimal glove design.
	\newblock {\em The International Journal of Robotics Research}.
	\newblock Submitted.

	\bibitem[Bianchi et~al., 2012b]{Bianchi_etalI}
	Bianchi, M., Salaris, P., and Bicchi, A. (2012b).
	\newblock Synergy-based hand pose sensing: Performance enhancement.
	\newblock {\em The International Journal of Robotics Research}.
	\newblock Submitted.

	\bibitem[Bicchi, 1992]{BicchiOpt}
	Bicchi, A. (1992).
	\newblock A criterion for optimal design of multiaxis force sensors.
	\newblock {\em Journal of Robotics and Autonomous Systems}, 10(4):269--286.

	\bibitem[Bicchi and Canepa, 1994]{Bicchican}
	Bicchi, A. and Canepa, G. (1994).
	\newblock Optimal design of multivariate sensors.
	\newblock {\em Measurement Science and Technology (Institute of Physics Journal
	  ``E'')}, 5:319--332.

	\bibitem[Chaloner and Verdinelli, 1995]{Bayesian1}
	Chaloner, K. and Verdinelli, I. (1995).
	\newblock Bayesian experimental design: A review.
	\newblock {\em Statistical Science}, 10:273--304.

	\bibitem[Chang et~al., 2007]{Chang}
	Chang, L.~Y., Pollard, N.~S., Mitchell, T.~M., and Xing, E.~P. (2007).
	\newblock Feature selection for grasp recognition from optical markers.
	\newblock In {\em Intelligent Robots and Systems, 2007. IROS 2007. IEEE/RSJ
	  International Conference on}, pages 2944--2950.

	\bibitem[Diamantaras and Hornik, 1993]{Kostas93}
	Diamantaras, K. and Hornik, K. (1993).
	\newblock Noisy principal component analysis.
	\newblock {\em Measurement`93}, pages 25 -- 33.

	\bibitem[Edin and Abbs, 1991]{Edin}
	Edin, B.~B. and Abbs, J.~H. (1991).
	\newblock Finger movement responses of cutaneous mechanoreceptors in the dorsal
	  skin of the human hand.
	\newblock {\em Journal of neurophysiology}, 65(3):657--670.

	\bibitem[Edmison et~al., 2002]{Edmin}
	Edmison, J., Jones, M., Nakad, Z., and Martin, T. (2002).
	\newblock Using piezoelectric materials for wearable electronic textiles.
	\newblock In {\em Wearable Computers, 2002. (ISWC 2002). Proceedings. Sixth
	  International Symposium on}, pages 41 -- 48.

	\bibitem[Gabiccini and Bicchi, 2010]{Gabicciniart}
	Gabiccini, M. and Bicchi, A. (2010).
	\newblock On the role of hand synergies in the optimal choice of grasping
	  forces.
	\newblock In {\em Robotics Science and Systems}.

	\bibitem[Ghosh and Rao, 1996]{Bayesian2}
	Ghosh, S. and Rao, C.~R. (1996).
	\newblock Review of optimal bayes designs.
	\newblock In {\em Design and Analysis of Experiments}, volume~13 of {\em
	  Handbook of Statistics}, pages 1099 -- 1147. Elsevier.

	\bibitem[Helmicki et~al., 1991]{Helmicki}
	Helmicki, A.~J., Jacobson, C.~A., and Nett, C.~N. (1991).
	\newblock Control oriented system identification: a worst-case/deterministic
	  approach in h$^{\infty}$.
	\newblock {\em Automatic Control, IEEE Transactions on}, 36(10):1163 --1176.

	\bibitem[Pukelsheim, 2006]{BookOptimalDesign}
	Pukelsheim, F. (2006).
	\newblock {\em Optimal Design of Experiments (Classics in Applied Mathematics)
	  (Classics in Applied Mathematics, 50)}.
	\newblock Society for Industrial and Applied Mathematics, Philadelphia, PA,
	  USA.

	\bibitem[Rao, 1964]{Rao64}
	Rao, C.~R. (1964).
	\newblock The use and interpretation of principal component analysis in applied
	  research.
	\newblock {\em The Indian journal of statistic}, 26:329 -- 358.

	\bibitem[Rosen, 1960]{Rosen}
	Rosen, J.~B. (1960).
	\newblock The gradient projection method for nonlinear programming. part i.
	  linear constraints.
	\newblock {\em Journal of the Society for Industrial and Applied Mathematics},
	  8(1):181 -- 217.

	\bibitem[Santello et~al., 1998]{Santelloart}
	Santello, M., Flanders, M., and Soechting, J.~F. (1998).
	\newblock Postural hand synergies for tool use.
	\newblock {\em The Journal of Neuroscience}, 18(23):10105 -- 10115.

	\bibitem[Sturman and Zeltzer, 1993]{Sturman93}
	Sturman, D.~J. and Zeltzer, D. (1993).
	\newblock A design method for ``whole-hand'' human-computer interaction.
	\newblock {\em ACM Trans. Inf. Syst.}, 11(3):219--238.

	\bibitem[Tempo, 1988]{Tempo}
	Tempo, R. (1988).
	\newblock Robust estimation and filtering in the presence of bounded noise.
	\newblock {\em Automatic Control, IEEE Transactions on}, 33(9):864 --867.

	\bibitem[Tognetti et~al., 2006]{Tognetti}
	Tognetti, A., Carbonaro, N., Zupone, G., and \mbox{De Rossi}, D. (2006).
	\newblock Characterization of a novel data glove based on textile integrated
	  sensors.
	\newblock In {\em Annual International Conference of the IEEE Engineering in
	  Medicine and Biology Society, EMBC06, Proceedings.}, pages 2510 -- 2513.

	\bibitem[Zavlanos and Pappas, 2008]{Pappas}
	Zavlanos, M.~M. and Pappas, G.~J. (2008).
	\newblock A dynamical systems approach to weighted graph matching.
	\newblock {\em Automatica}, 44(11):2817 -- 2824.
 	
\end{thebibliography}

\appendix
\section{Appendix}
This appendix is devoted to the derivation of the gradient equation given in proposition~\ref{thm:FinalGradientFlow}.

\paragraph*{\textbf{Proof of Proposition 1}}
The Frobenius norm of a matrix $A\in\real^{n\times n}$ is given as
\[
\|A\|_F = \sqrt{\tr(A^TA)}= \sqrt{\sum_{i=1}^n\sigma_i^2}\,,
\]
and hence,
\begin{equation}
	\label{eq:Problem}
\|P_o-P_oH^T(HP_oH^T+R)^{-1}HP_o\|^2_F=\tr(P_p^TP_p)
\end{equation}
where $P_p = P_o-P_oH^T(HP_oH^T+R)^{-1}HP_o$. To find the gradient flow, we need to compute
\begin{align}
\frac{\partial \tr(P_p^TP_p)}{\partial H} &= \tr\left(\frac{\partial (P_p^TP_p)}{\partial H}\right) = \tr\left(\frac{\partial P_p^T}{\partial H}\,P_p+P_p^T\,\frac{\partial P_p}{\partial H}\right) =\nonumber\\
&= \tr\left(\frac{\partial P_p^T}{\partial H}\,P_p\right)+\tr\left(P_p^T\,\frac{\partial P_p}{\partial H}\right) = 2\tr\left(P_p^T\,\frac{\partial P_p}{\partial H}\right)\,,
\label{eq:DerivativeStep1}
\end{align}
as $\partial (\mathbf{X}\mathbf{Y})=(\partial\mathbf{X})\mathbf{Y}+\mathbf{X}(\partial\mathbf{Y})$ and $\tr(A^T)=\tr(A)$. Moreover, from differentiation rules of expressions w.r.t.~a matrix $\mathbf{X}$, we have
$\partial \mathbf{X}^{-1} = -\mathbf{X}^{-1}(\partial\mathbf{X}) \mathbf{X}^{-1}$
and hence, assuming $\Sigma(H) = (HP_oH^T+R)^{-1}$, we obtain
\begin{align}
	\frac{\partial P_p}{\partial\mathbf{H}} &= -P_o\left[(\partial H)^T\Sigma(H)H+H^T\left(\frac{\partial\Sigma(H)}{\partial H}H+\Sigma(H)\,\partial\, H\right)\right]P_o =\nonumber\\
	&= -P_o\left[(\partial H)^T\Sigma(H)H-H^T\left(\Sigma(H)\left(\partial HP_oH^T+\right.\right.\right.\nonumber\\
	&\left.\left.\left.+HP_o(\partial H)^T\right)\Sigma(H)H+\Sigma(H)\,\partial H\right)\right]P_o\,.
	\label{eq:DerivativeStep2}
\end{align}
Substituting \eqref{eq:DerivativeStep2} in \eqref{eq:DerivativeStep1} and by using a well note trace property ($\tr(A+B)=\tr(A)+\tr(B)$) we obtain
\begin{align}
	\frac{\partial \tr(P_p^TP_p)}{\partial H} &= 2\left[-\tr(P_p^TP_o(\partial H)^T\Sigma(H)HP_o)+\tr(P_p^TP_oH^T\Sigma(H)\partial HP_oH^T\Sigma(H)HP_o) +\right.\nonumber\\
	&\left.+\tr(P_p^TP_oH^T\Sigma(H)HP_o(\partial H)^T\Sigma(H)HP_o)-\tr(P_p^TP_oH^T\Sigma(H)\partial HP_o)\right]\,.
\end{align}
As $\tr(AB)=\tr(BA)$, we obtain
\begin{align}
	\frac{\partial \tr(P_p^TP_p)}{\partial H} &= 2\left[-\tr((\partial H)^T\Sigma(H)HP_oP_p^TP_o)+\tr(P_oH^T\Sigma(H)HP_oP_p^TP_oH^T\Sigma(H)\partial H)+\right.\nonumber\\
	&\left.+\tr((\partial H)^T\Sigma(H)HP_oP_p^TP_oH^T\Sigma(H)HP_o)-\tr(P_oP_p^TP_oH^T\Sigma(H)\partial H)\right]
\end{align}
and as $\tr(A^T)=\tr(A)$ we have
\begin{align}
	\frac{\partial \tr(P_p^TP_p)}{\partial H} &= 2\left[-\tr(P_o^TP_pP_o^TH^T\Sigma(H)^T\partial H)+\tr(P_oH^T\Sigma(H)HP_oP_p^TP_oH^T\Sigma(H)\partial H)+\right.\nonumber\\
	&\left.+\tr(P_o^TH^T\Sigma(H)^THP_o^TP_pP_o^TH^T\Sigma(H)^T\partial H)-\tr(P_oP_p^TP_oH^T\Sigma(H)\partial H)\right]\,,
\end{align}
whence,
\begin{align}
	\frac{\partial \tr(P_p^TP_p)}{\partial H} &= 2\left[-P_o^TP_pP_o^TH^T\Sigma(H)^T+P_oH^T\Sigma(H)HP_oP_p^TP_oH^T\Sigma(H)+\right.\nonumber\\
	&\left.+P_o^TH^T\Sigma(H)^THP_o^TP_pP_o^TH^T\Sigma(H)^T-P_oP_p^TP_oH^T\Sigma(H)\right] =\nonumber\\
	&= 2\left[(P_oH^T\Sigma(H)H-I)P_oP_p^TP_oH^T\Sigma(H)+(P_o^TH^T\Sigma(H)^TH-I)P_o^TP_pP_o^TH^T\Sigma(H)^T\right]\,.
\end{align}
Matrices $P_p$, $P_o$ and $\Sigma(H)$ are symmetric, and hence, for this particular case we obtain
\begin{equation}
	\frac{\partial \tr(P_p^TP_p)}{\partial H} = -4\left[P_p^2P_oH^T\Sigma(H)\right]^T\,,
	\label{eq:GradMA}
\end{equation}
with $\Sigma(H) = (HP_oH^T+R)^{-1}$.
\end{document}